%% file: main.tex
\documentclass[sigconf]{acmart}

\AtBeginDocument{%
  \providecommand\BibTeX{{%
    \normalfont B\kern-0.5em{\scshape i\kern-0.25em b}\kern-0.8em\TeX}}}

\usepackage{balance}  
\usepackage{stmaryrd}
\usepackage{booktabs} 
\usepackage{framed}

\usepackage[linesnumbered,boxed,ruled]{algorithm2e}
\usepackage{enumitem}
\usepackage{dsfont}
\usepackage{graphicx, subfigure}
\usepackage{multirow}

\newcommand{\nosection}[1]{\vspace{1.5pt}\noindent\textbf{#1.}}

\newcommand{\cdrmodel}{HeteroCDR}
\newcommand{\modelname}{PriCDR}

\newcommand{\GN}{\mathcal{N}}
\newcommand{\M}{\mathcal{M}}
\newcommand{\x}{\mathbf{x}}

\newcommand{\E}{\mathbb{E}}
\newcommand{\pdf}{\rm{PDF}}

\copyrightyear{2022} 
\acmYear{2022} 
\setcopyright{acmcopyright}\acmConference[WWW '22]{Proceedings of the ACM Web Conference 2022}{April 25--29, 2022}{Virtual Event, Lyon, France}
\acmBooktitle{Proceedings of the ACM Web Conference 2022 (WWW '22), April 25--29, 2022, Virtual Event, Lyon, France}
\acmPrice{15.00}
\acmDOI{10.1145/3485447.3512192}
\acmISBN{978-1-4503-9096-5/22/04}

\begin{document}
\title{Differential Private Knowledge Transfer for Privacy-Preserving Cross-Domain Recommendation}\thanks{$^*$Corresponding author.}

\author{Chaochao Chen$^{1}$, Huiwen Wu$^{2,*}$, Jiajie Su$^{1}$, Lingjuan Lyu$^{3}$, Xiaolin Zheng$^{1,4}$ and Li Wang$^{2}$}
\affiliation{
  $^{1}$Zhejiang University, $^{2}$Ant Group, $^{3}$Sony AI, $^{4}$JZTData Technology\\
\{zjuccc, sujiajie, xlzheng\}@zju.edu.cn, \{huiwen.whw, raymond.wangl\}@antgroup.com, Lingjuan.Lv@sony.com
}



\renewcommand{\shortauthors}{Chaochao Chen et al.}


\begin{abstract}
Cross Domain Recommendation (CDR) has been popularly studied to alleviate the cold-start and data sparsity problem commonly existed in recommender systems. 
CDR models can improve the recommendation performance of a target domain by leveraging the data of other source domains. 
However, most existing CDR models assume information can directly `transfer across the bridge', ignoring the privacy issues. 
To solve this problem, we propose a novel two stage based privacy-preserving CDR framework (\modelname). 
In the first stage, we propose two methods, i.e., Johnson-Lindenstrauss Transform (JLT) and Sparse-aware JLT (SJLT), to publish the rating matrix of the source domain using Differential Privacy (DP). 
We theoretically analyze the privacy and utility of our proposed DP based rating publishing methods. 
In the second stage, we propose a novel heterogeneous CDR model (\cdrmodel), which uses deep auto-encoder and deep neural network to model the published source rating matrix and target rating matrix respectively. 
To this end, \modelname~can not only protect the data privacy of the source domain, but also alleviate the data sparsity of the source domain. 
We conduct experiments on two benchmark datasets and the results demonstrate the effectiveness of  \modelname~and \cdrmodel.
\end{abstract}


\maketitle

\input{sections/intro}
\input{sections/relatedwork}

\input{sections/preliminary}

\input{sections/model} 
\input{sections/analysis1}
\input{sections/experiment}

\input{sections/conclusion}
\begin{acks}
This work was supported in part by the National Key R\&D Program of China (No.2018YFB1403001) and the National Natural Science Foundation of China (No. 62172362 and No. 72192823).
\end{acks}

\bibliographystyle{ACM-Reference-Format}
\bibliography{reference}

\newpage
\appendix
\input{sections/appendix1}

\end{document}

%% file: sections/intro.tex
\section{Introduction}
\label{sec:intro}

Cross-Domain Recommendation (CDR) has been popularly studied recently, due to its powerful ability to solve the data sparsity issue of traditional recommender systems \cite{zhu2021cross}. 
Most existing CDR models focus on leveraging different kinds of data across multiple domains to improve the recommendation performance of a target domain (or domains) through `bridges'.  
Here, the data are usually users' private data, e.g., user-item rating, review, and user profile. 
The `bridge' refers to the linked objects between different domains such as common user, common item, and common features. 


\nosection{Privacy issues in CDR}
Existing CDR models assume information can directly `transfer across the bridge', ignoring the privacy issues. 
%
The key of CDR model is designing transfer strategies across bridges, e.g., cross network \cite{hu2018conet} and dual  transfer network \cite{li2020ddtcdr}. 
%
The above models assume all data are plaintext to both domains, which is inconsistent with reality. 
Since these data for CDR models usually involve user sensitive information and cannot be shared with others due to regulation reasons. 
Therefore, how to build CDR models on the basis of protecting data privacy becomes an urgent research problem. 
 
\nosection{Problem definition}
In this paper, we consider a two-domain cross recommendation problem for simplification. 
Specifically, we assume there are two domains (a source domain and a target domain) which have the same user set but different user-item interaction pairs. 
We also assume the user-item interactions only have rating data and leave other side-information as future work. 
The problem of privacy-preserving CDR is to improve the recommendation performance of the target domain by leveraging the data of the source domain, while protecting the data privacy of both domains.

\nosection{Existing work}
The existing privacy-preserving recommendation models are mainly designed for two categories, i.e., individual customers and business partners. 
The former assumes that each client only contains the private data of a single user \cite{nikolaenko2013privacy,hua2015dp,chen2018private}, and the existing work on this either assumes the existence of a semi-honest server or reveals middle information (e.g., latent factors) during model training. 
The later assumes that each client (domain) has the private data of a batch of users \cite{ogunseyi2021privacy,gao2019privacy,gao2019cross,chen2019secure,cui2021exploiting}, which can be seen as privacy-preserving CDR models. 
%
However, they are customized to certain settings, e.g., social recommendation \cite{chen2019secure} and  location recommendation \cite{gao2019privacy}, and thus are not suitable to our problem. 


\nosection{Our proposal}
In this paper, we propose a novel two stage based privacy-preserving CDR framework (\modelname). 
In \textit{stage one}, the source domain privately publishes its user-item ratings to the target domain. 
In \textit{stage two}, the target domain builds a CDR model based on its raw data and the published data of the source domain.

\textit{Design goals.} 
The two stages in \modelname~should have the following goals. 
First, for stage one, user-item rating matrix publishing should not only preserve the data privacy of the source domain, but also facilitate the following CDR task. 
Specifically, there should be the following requirements in rating publishing. 
(1) \textit{Privacy-preserving.} From the published user-item ratings, the target domain should not identify whether a user has rated an item. 
%
(2) \textit{Restricted isometry property.}
The nature of most recommender systems is collaborative filtering. That is, users who have similar ratings tend to share similar tastes. 
%
Thus, the published rating matrix should approximate the original one well, as the restricted isometry property stated
in the literature of matrix completion~\cite{candes2005decoding}. Thus, the users can still \textit{collaborate} with each other. 
(3) \textit{Sparse-awareness.} Data sparsity is a long-standing problem in recommender systems. 
Thus, rating publishing should be able to handle the sparse user-item rating data. 
%
After rating publishing in stage one, the user-item rating data becomes heterogeneous in stage two. 
However, most existing CDR frameworks, e.g., \cite{zhu2019dtcdr,etl,darec19,li2020ddtcdr,hu2018conet}, are symmetrical. 
Thus, a new CDR framework that can handle such heterogeneous data is needed.



\textit{Technical solution.} 
In stage one, inspired by the fact that Johnson-Lindenstrauss Transform (JLT) preserves differential privacy \cite{blocki2012johnson}, we propose two methods to publish rating matrix differential privately. 
We first propose JLT based differential private rating matrix publishing mechanism, which can not only protect data privacy but also preserve the similarity between users. 
We then propose Sparse-aware JLT (SJLT) which can further reduce the computation complexity of JLT and alleviate the data sparsity in the source domain with sub-Gaussian random matrix and Hadamard transform. 
We finally theoretically analyze their privacy and utility. 
In stage two, we propose a novel heterogeneous CDR model (\cdrmodel), which uses deep auto-encoder and deep neural network to model the published source rating matrix and target rating matrix, respectively. 
We also propose an embedding alignment module to align the user embeddings learnt from two networks. 

\nosection{Contributions}
We summarize our main contributions as follows: 
(1) We propose \modelname, a general privacy-preserving CDR framework, which has two stages, i.e., rating publishing and cross-domain recommendation modeling. 
(2) We propose two differential private rating publishing algorithms, via JLT and SJLT respectively, we also provide their privacy and utility guarantees. 
(3) We propose \cdrmodel~to handle the heterogeneous rating data between the published source domain and target domain. 
(4) We conduct experiments on two benchmark datasets and the results demonstrate the effectiveness of \modelname~and \cdrmodel.

%% file: sections/relatedwork.tex
\section{Related Work}
\label{sec:rw}

\subsection{Cross Domain Recommendation}
Cross Domain Recommendation (CDR) is proposed to
handle the cold-start and data sparsity problem commonly existed in the traditional single domain recommender systems \cite{zhu2021cross}. 
The basic assumption of CDR is that different behavioral patterns from multiple domains jointly characterize
the way users interact with items \cite{pan21dual,zhu2019dtcdr}. 
Thus, CDR models can improve the recommendation performance of the single domain recommender systems by transferring the knowledge learnt from other domains using auxiliary information. 
To date, different kinds of knowledge transferring strategies are proposed in CDR, e.g.,  cross
connections \cite{hu2018conet}, dual learning \cite{li2020ddtcdr}, and adversarial training \cite{darec19}. 
However, most existing CDR models assume information can directly `transfer across the bridge', ignoring the privacy issues. 
In this paper, we aim to solve the privacy issue in CDR using differential privacy. 

\subsection{Privacy-Preserving Recommendation}

Most existing privacy-preserving recommendation models are built to protect the data privacy of individual customers (2C) rather than business partners (2B). 
The main difference between  these two types of models is the data distribution on clients. 
The former (2C case) assumes that each client only contains the private data of a single user, while the later (2B case) assumes that each client has the private data of a batch of users. 
The representative work of the 2C models include 
\cite{mcsherry2009differentially,ziqidpmf,nikolaenko2013privacy,hua2015dp,chen2018private}, which use different types of techniques. 
For example, McSherry et al.~\cite{mcsherry2009differentially} presented a differential private recommendation method by adding plain Gaussian noise on the covariance matrix. 
%
%
Hua et al. \cite{hua2015dp} adopted differential privacy for matrix factorization, and there is a trade-off between privacy and accuracy. 
Chen et al. \cite{chen2018private} proposed decentralized matrix factorization and can only applied for point-of-interest recommendation. 
%
%
The representative work of the 2B model includes \cite{chen2019secure,gao2019privacy,liu2021fedct,ogunseyi2021privacy,cui2021exploiting}. 
First, the model in \cite{chen2019secure} was designed for secure social recommendation, which is not suitable to the situations where all the domains have user-item rating data. 
Second, \cite{gao2019privacy} was designed for location recommendation, which uses differential privacy to protect user location privacy and is not suitable to our problem. 
Besides, \cite{ogunseyi2021privacy} uses somewhat homomorphic encryption and thus 
is limited to only small-scale dataset (with hundreds of users). 

In summary, the existing privacy-preserving recommendation approaches are either designed for 2C scenario or 2B scenarios that are not suitable to our problem. 
In this paper, we propose to adopt differential privacy for 2B scenario where two domains have the same batch of users but different sets of items.

%% file: sections/preliminary.tex
\section{Preliminaries}
\label{sec:pre}


\subsection{User-item Rating Differential Privacy}
Differential privacy is a robust, meaningful, and mathematically rigorous definition of privacy~\cite{dwork2014algorithmic}. 
The goal of differential privacy is to learn information about the population as a whole while protecting the privacy of each individual~\cite{dwork2011firm, dwork2007ad}. 
By a careful design mechanism, differential privacy controls the fundamental quantity of information that can be revealed with changing one single individual ~\cite{jiang2016wishart}.
Before diving into the differential private mechanism for user-item rating matrix publishing, we first introduce some basic knowledge on differential privacy. 




In the scenario of recommender systems, we have a rating matrix $\mathbf{R}$ as defined in Definition~\ref{def:uirate}. 
Under such setting, we also define the neighbouring rating matrices in Definition~\ref{def:neigh_rating}. 
%

\begin{definition}[User-item Rating Matrix]
Let $\mathbf{U} = \{\mathbf{u}_1, \cdots, \mathbf{u}_m \}$ be a set of users
and $\mathbf{V} = \{\mathbf{v}_1, \cdots, \mathbf{v}_n \}$ be a set of items. 
Define user-item rating matrix $\mathbf{R} = \begin{bmatrix} r_{ij} \end{bmatrix}$ with $r_{ij} \geq 0$ denoting the rating of user $i$ on item $j$. 
\label{def:uirate}
\end{definition}




\begin{definition}[Neighbouring Rating Matrices]
Two rating matrices $\mathbf{R}$ and $\mathbf{R}'$ are neighbouring if exactly one user-item rating in $\mathbf{R}$ is changed arbitrarily to obtain $\mathbf{R}'$. 
Suppose $\mathbf{R} = \begin{bmatrix} r_{ij} \end{bmatrix}_{n \times m}$ and $\mathbf{R}' = \begin{bmatrix} r_{ij}' \end{bmatrix}_{n \times m}$, there exists one pair $(i_0, j_0)$ with $1< i_0 < n$ and $1 < j_0 < m$ such that 
$$
\begin{cases}
r_{ij} \neq r_{ij}'\quad \text{and} \quad |r_{ij} - r_{ij}'| <1,  \quad & i = i_0, \quad j = j_0; \\
r_{ij} = r_{ij}',  \quad & i \neq i_0, \quad j \neq j_0. 
\end{cases}
$$
~\label{def:neigh_rating}
\end{definition}
%

To measure the privacy leakage of a randomized algorithm $\M$, we define the privacy loss on neighbouring rating matrices below. 

\begin{definition}[Privacy Loss~\cite{dwork2014algorithmic}]
Let $\mathbf{R}$ and $\mathbf{R}'$ be two neighbouring rating matrices, 
$\M$ be a randomized algorithm, $\M(\mathbf{R})$ and $\M(\mathbf{R}')$ be the probability distributions induced by $\M$ on $\mathbf{R}$ and $\mathbf{R}'$ respectively. 
The privacy loss is a random variable defined as 
$$
\mathcal{L}_M^{\mathbf{R},\mathbf{R}'}(\theta)\doteq\log(P(\mathcal{M}(\mathbf{R})=\theta)/P(\mathcal{M}({\mathbf{R}'})
=\theta).
$$
where $\theta \sim \M(\mathbf{R})$,
\label{def:privacy_loss}
\end{definition}

We want to protect the privacy of each user-item rating pair ($r_{ij}$), i.e., by changing a single $r_{ij}$ such that the output result of $\M$ is almost indistinguishable. 
%
%
Thus, we define the differential privacy with respect to user-item rating matrix $\mathbf{R}$ in Definition~\ref{def:rating_dp}.

\begin{definition}[User-level Differential Privacy] 
A randomized algorithm $\M$ that takes a user-item vector as input is $(\epsilon, \delta)-$differential private if for any differ-by-one user-item vectors $\mathbf{v}, \mathbf{v}'$, and any event $\mathcal{A}$, 
\begin{equation}
P [\mathcal{M}(\mathbf{v}) \in \mathcal{A} ] \leq \exp(\epsilon) P[\mathcal{M}(\mathbf{v}') \in \mathcal{A}] + \delta. 
\end{equation}
 If $\delta = 0$, we say that $\M$ is $\epsilon-$differentially private. \label{def:user_rating_dp}
\end{definition}

\begin{definition}[Rating Matrix Differential Privacy]
A randomized algorithm $\M$ that takes rating matrix $\mathbf{R}$ as input guarantees $(\epsilon, \delta)-$differential privacy (DP) if for any pair of neighbouring rating matrices $\mathbf{R}, \mathbf{R}'$, and any event $\mathcal{A}$, 
\begin{equation}
P [\mathcal{M}(\mathbf{R}) \in \mathcal{A} ] \leq \exp(\epsilon) P[\mathcal{M}(\mathbf{R}') \in \mathcal{A}] + \delta. 
\end{equation}
 If $\delta = 0$, we say that $\M$ is $\epsilon-$differentially private. \label{def:rating_dp}
\end{definition}




\subsection{Random Transform}
We present two random transforms which preserve DP. 


\begin{definition}[Johnson-Lindenstrauss Transform (JLT)~\cite{blocki2012johnson}]
The Johnson-Lindenstrauss that transforms a vector $\x$ from $\mathbb{R}^d$ to $\mathbb{R}^{d'}$ is defined by $\Phi: \x \rightarrow \mathbf{M} \x$, where $\mathbf{M} \sim \GN(0, \mathbf{O}_{d' \times d})$, where $\mathbf{O}_{d' \times d}$ is a dense matrix with all elements equal one. 
\label{def:jlt}
\end{definition}

Sparse-aware JLT performs similarly as JLT, which maps a high-dimensional space to a lower-dimensional space with $\ell_2-$norm of transformed vector maintained with a significant probability. 

\begin{definition}[Sparse-aware JLT (SJLT)~\cite{ailon2009fast}]
The SJLT that transforms a sparse vector from $\mathbb{R}^d$ to $\mathbb{R}^{d'}$ is defined by $\Psi: \x \rightarrow \mathbf{M} \x$, where $\mathbf{M} = \mathbf{P H D}$ with $\mathbf{P} \in \mathbb{R}^{d' \times d}$, $\mathbf{H} \in \mathbb{R}^{d \times d}$, $\mathbf{D} \in \mathbb{R}^{d \times d}$
\begin{equation}\label{eq:def-p}
\mathbf{P}_{ij} = 
\begin{cases}
0 \quad & \text{with probability }\quad 1-q; \\
\xi \sim \GN(0, q^{-1}) \quad & \text{with probability }\quad q;
\end{cases}
\end{equation}
$
\mathbf{H}_{ij} = d^{-1/2} (-1)^{(i-1, j-1)}
$
is the normalized Hadamard matrix where $(i, j) = \sum_k i_k j_k \mod 2$ is the bit-wise inner product mod 2, and $\mathbf{D}$ is a random diagonal matrix where 
\begin{equation}\label{eq:def-d}
\mathbf{D}_{ii} = \pm 1 \quad \text{with equal probability.}
\end{equation}

\label{def:fjlt}
\end{definition}

%% file: sections/model.tex
\section{The Proposed Framework}
\label{sec:model}

In this section, we first describe the notations and then present the framework of \modelname~and its two stages. 


\subsection{Notations}
\label{sec:model:notation}
We assume there are only two domains, i.e., domain $\mathcal{A}$ and domain $\mathcal{B}$, which have the same batch of users but different user-item pairs. 
We assume domain $\mathcal{A}$ has a private rating matrix $\mathbf{R}^A \in\mathbb{R}^{m\times n_1}$, and domain $\mathcal{B}$ holds another private rating matrix $\mathbf{R}^B \in\mathbb{R}^{m\times n_2}$, where $m$ denotes the number of users, $n_1$ and $n_2$ denote the number of items on domain $\mathcal{A}$ and domain $\mathcal{B}$, respectively. 
%
We use $\mathbf{U}_i^A$ and $\mathbf{U}_i^B$ to denote the latent embedding of the user $i$ on the published source domain and target domain, respectively. 
We also use $\mathbf{V}_j^B$ to denote the latent embedding of item $j$ on the target domain. 
%
%
Notice that for a matrix $\mathbf{X}$, we denote its $i$-th column as $\mathbf{x}_{i}$.
We summarize the main notations in Appendix~\ref{sec:app:note}.



\begin{figure}[t]
 \begin{center}
 \includegraphics[width=\columnwidth]{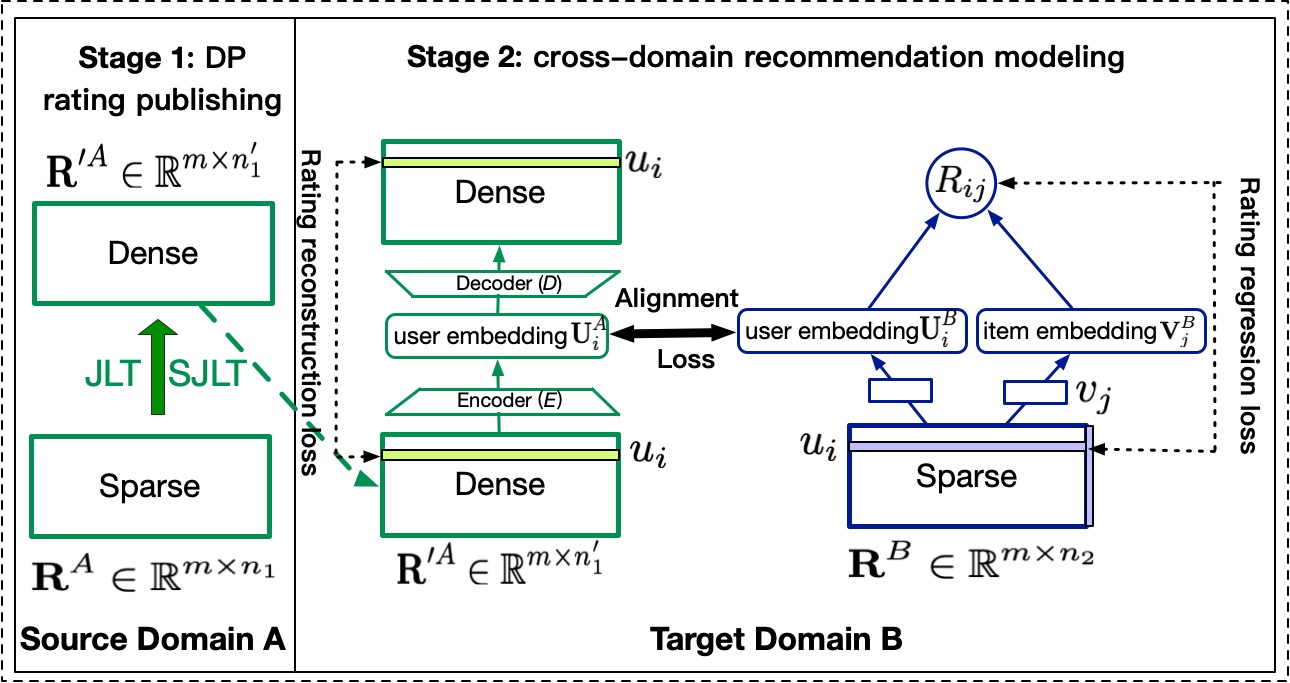}
 \vspace{-0.2 in}
 \caption{Framework of \modelname.}
 \vspace{-0.15 in}
 \label{fig:framework}
 \end{center}
\end{figure}

\subsection{Overview of \modelname}
We propose \modelname, a general privacy-preserving CDR framework. 
Similar as most existing CDR models, \modelname~aims to transfer the knowledge learnt from the source domain to help improve the recommendation performance of the target domain. 
However, different from the existing CDR models that assume the data from both the source and target domains are available, \modelname~aims to build a CDR model on the basis of protecting the data of both domains.  

\modelname~consists of two stages, i.e., \textit{rating publishing} and \textit{cross-domain recommendation modeling}, as is shown in Figure \ref{fig:framework}.
First, as described in Section \ref{sec:intro}, rating publishing mainly has three goals, i.e., \textbf{G1}: \textit{privacy-preserving}, \textbf{G2}: \textit{restricted isometry property}, and \textbf{G3}: \textit{sparse-aware}. 
To achieve these goals, we propose two differential private rating matrix publishing algorithms, including JLT and SJLT. 
The former one achieves \textit{privacy-preserving} and \textit{restricted isometry property} while the later one further achieves \textit{sparse-aware}, which will be further demonstrated both in theoretically in Sec.~\ref{sec:ana} and experimentally in Sec.~\ref{sec:experi}. 
Second, to handle the heterogeneity between the published rating matrix of the source domain and the original rating matrix of the target domain, we propose a new cross-domain recommendation model (i.e., \cdrmodel). 
~\cdrmodel~has three parts, a \textit{rating reconstruction module} for the published source domain, a \textit{rating regression module} for the target domain, and a \textit{user embedding alignment module} between the user embeddings of the source domain and target domain. 

With our proposed two stages, \modelname~can not only protect rating privacy, but also take advantage of the characteristics of both domains. 
Thus, \modelname~can achieve comparable even better performance than most existing plaintext CDR models, as will be reported in Sec.~\ref{sec:experi}. 
We will present these two stages in details. 

\begin{algorithm}[t]
\caption{Differential Private Rating Matrix Publishing }\label{alg:dp_rating_pub}
\KwIn{ A rating matrix $\mathbf{R}$; 
parameters: $\epsilon, \delta, \eta, \mu > 0, q = O(\frac{1}{m}) \in (0,1)$; 
number of users $m$;
number of items in the source domain: $n_1$ }
\KwOut{ Privacy preserving rating matrix $\mathbf{\tilde{R}}$ }

{
\# \textbf{Perturb the singular values of normalized $\mathbf{R}$} \\\label{alg:dp_rating_pub:begin}
Subtract the mean from $\mathbf{R}$ by computing 
$r_{ij} \rightarrow r_{ij} - \frac{1}{m}\sum_{k=1}^m r_{ik};$
\\ \label{alg:dp_rating_pub:submean}
Compute $n'_1 = {8 \ln(2/\mu)}/{(\eta^2)}$ and $w = \sqrt{32 n'_1 \ln(2/\delta)}/{\epsilon} \ln(4 n'_1 /\delta);$\\
Compute the SVD of $\mathbf{R} = \mathbf{U \Sigma V^{\intercal}}$; \\
Set $\mathbf{R} \leftarrow \mathbf{U} \sqrt{\mathbf{\Sigma}^2 + w^2 \mathbf{I}} \mathbf{V}^{\intercal},$ with $\mathbf{I}$ a identity matrix. \\ \label{alg:dp_rating_pub:perfinal}
\# \textbf{Generate random matrix \textbf{M} by JLT}\\ \label{alg:dp_rating_pub:tranbegin}
Draw i.i.d. samples from $\GN(0,1)$, and form a matrix $\textbf{M} \in\mathbb{R}^{n'_1 \times n_1}$; \\
\# \textbf{Generate random matrix \textbf{M} by SJLT}\\
Generate a sparse Johnson-Lindenstrauss matrix $\mathbf{P}$ of size $n'_1 \times n_1$ using Equation \eqref{eq:def-p} \\
Generate a normalized Hadamard matrix of size $n_1 \times n_1$ with 
$
\mathbf{H}_{ij} \leftarrow {n_1}^{-1/2} (-1)^{(i-1, j-1)},
$
where $(i, j) = \sum_k i_k j_k \mod 2$ is the bit-wise inner product of index $i, j$; \\
Generate a diagonal matrix of size $n_1 \times n_1$ using Equation \eqref{eq:def-d}\\
Compute \textbf{M} by $\textbf{M} = \mathbf{PHD}$ \\
\# \textbf{For both JLT and SJLT} \\
Transform rating matrix $\mathbf{\tilde{R}} \leftarrow \frac{1}{\sqrt{n'_1}} \mathbf{M R};$ \\ \label{alg:dp_rating_pub:tranend}
\textbf{return} The perturbed matrix $\mathbf{\tilde{R}}$. \\
}
\end{algorithm} 

\subsection{Differential Private Rating Publishing}

We first present the rating publish stage of \modelname. 
We protect the source domain data by publishing the source domain rating matrix to the target domain utilizing differential privacy.
To protect the individual user-item rating privacy, we want the rating publishing stage to satisfy $(\epsilon, \delta)-$differential privacy. 
For this purpose, we propose to publish the rating matrix in the source domain differential privately via JLT or SJLT, as is shown in Algorithm~\ref{alg:dp_rating_pub}. 
The whole process of the rating publish stage can be divided into two procedures, i.e., \textit{perturb the singular values of the normalized rating matrix} (lines \ref{alg:dp_rating_pub:begin}-\ref{alg:dp_rating_pub:perfinal}) and \textit{transform the perturbed rating matrix by JLT or SJLT} (lines \ref{alg:dp_rating_pub:tranbegin}-\ref{alg:dp_rating_pub:tranend}). 
We now describe these procedures in details. 

%
\nosection{Explanation of hyper-parameters}
First of all, we introduce the hyper-parameters used Algorithm~\ref{alg:dp_rating_pub}. $\epsilon$ and $\delta$ are privacy-related hyper-parameters. We tune $\epsilon$ for different privacy budget and fix $\delta=1/n$, where $n$ is number of input data. 
$\mu$ and $\eta$ are utility-related hyper-parameters to determine the subspace dimension $n_1'$ after applying JLT or SJLT. The larger $\mu$ and $\eta$, the smaller $n_1'$, the smaller probability we get for the concentration results in Theorem~\ref{thm:RIP}. 

\nosection{Perturb the singular values of the normalized rating matrix}
Before applying random transforms on rating matrix, we apply singular value perturbation, which consists of four steps. 
The first step is to subtract the mean from rating matrix $\mathbf{R}$ (line \ref{alg:dp_rating_pub:submean}).
The second step is to compute the compressed dimension $n_1'$ and singular value perturbation $w$ (line 3). 
%
%
The third step is to compute the singular value decomposition of matrix $\mathbf{R}$ and get the singular values as the non-zero elements of the diagonal matrix $\Sigma$ (line 4). 
The fourth step is to modify the singular values of $\mathbf{R}$ by $\sqrt{\sigma_i^2 + w^2}$, where $\sigma_i = \Sigma_{ii}$ (line 5). 
By exerting proper singular value perturbation on the randomized rating matrix, we get a perturbed rating matrix with the same singular vectors but a different spectrum from the original rating matrix. 
In recommender systems, it is common to use the singular vectors corresponding to top singular values. 
Thus, we preserve the privacy by perturbing the singular values while maintaining utility by keeping singular vectors unchanged, which achieves \textbf{G1}: \textit{privacy-preserving}. 

\nosection{Transform the perturbed rating matrix by JLT or SJLT} 
After the first procedure, one can transform the perturbed rating matrix randomly via JLT (line 7 and line 14) or SJLT (line 9-12 and line 14). 
%
By applying random transform, we exert a proper randomness on rating matrix where the magnitude is determined by perturbation parameter $w$ in the previous procedure. 
By the expectation approximation result (Proposition~\ref{prop:mean_approx}) and 
concentration result of Gaussian or sub-Gaussian random matrix ( Proposition~\ref{prop:jl}), we conclude that JLT and SJLT achieve \textbf{G2:} \textit{restricted isometry property} ( see Theorem~\ref{thm:RIP}). 
The main difference between JLT and SJLT is how to generate the random matrix $\mathbf{M}$.
The random matrix $\mathbf{M}$ of JLT is generated by Gaussian ensemble, while that of SJLT is generated by sub-Gaussian ensemble and random Hadamard matrix, as described in Definition~\ref{def:fjlt}. 
By singular value perturbation and random transform with JLT or SJLT, we get a differentially privately published rating matrix $\mathbf{\tilde{R}}$. 
Compared with JLT, 
(1) SJLT can reduce the computation complexity of matrix multiplication by utilizing random Hardmard transform with computation complexity of $O(d \log d m)$ while JLT costs $O(d n m)$ (line 14).
(2) the random hardmard transform $\mathbf{HD}$ in SJLT serves as a precondition when input $x$ is sparse (verified by Lemma~\ref{lem:rhd}), 
and thus SJLT achieves \textbf{G3:} \textit{sparse-awareness}.

\subsection{CDR Modeling}

We then present the cross-domain modeling stage of \modelname. 
The cross-domain modeling stage is done by the target domain by using the published rating matrix of the source domain and the original rating matrix of the target domain. 

We propose a novel cross-domain recommendation model, i.e., \cdrmodel, to solve the heterogeneity between the published rating matrix of the source domain and the original rating matrix of the target domain. 
\cdrmodel~has three parts: (1) a rating reconstruction module for the differential privately published source domain; (2) a rating regression module for the target domain; and (3) a user embedding alignment module between the user embeddings of the source domain and target domain. 
We will describe these three modules separately in detail. 

\nosection{Rating reconstruction of the source domain}
Let $\mathbf{R}^A \in\mathbb{R}^{m\times n_1}$ be the private rating matrix of the source domain, $\textbf{R}'^A \in\mathbb{R}^{m\times n'_1}$ be the published rating matrix of the source domain, and $n'_1 < n_1$. 
Considering that $\mathbf{R}'^A$ is a dense matrix with each row denoting the perturbed behaviors of each user, we use auto-encoder to learn user embedding. 
Specifically, we can first obtain the user embedding by the encoder, i.e., $\mathbf{U}_i^A=E(\mathbf{R}'^A_i) \in\mathbb{R}^{m\times h}$, where $h$ is the dimension of user embedding. 
The autoencoder can then reconstruct the perturbed ratings of each user by $\hat{\mathbf{R}'}^A_i = D(\mathbf{U}_i^A) \in\mathbb{R}^{m\times n'_1}$.
After it, the reconstruction loss is given by:
\begin{equation}
\begin{aligned}
\label{equ:recon}
\mathcal{L}_{rec} = \sum_{i=1}^m  \mathcal{F}(\mathbf{R}'^A_i, \hat{\mathbf{R}'}^A_i) ,
\end{aligned}
\end{equation}
where $\mathcal{F}(\cdot,\cdot)$ is the mean square loss.
With the rating reconstruction module, the encoder and decoder can model the perturbed user-item rating interactions in the source domain. 

\nosection{Rating regression of the target domain}
The rating regression module in the target domain is similar as the existing deep neural network based recommendation models, e.g., Deep Matrix Factorization (DMF) \cite{xue2017deep} and Neural Matrix Factorization (NeuMF) \cite{he2017neural}. 
In this paper, we take DMF as an example. 
The key idea of DMF is to minimize the cross-entropy between the true ratings $\textbf{R}^B_{ij}$ and the predicted ratings $\hat{\mathbf{R}}^B_{ij}$, where the predicted ratings is obtained by multiple fully-connected layers. 
Specifically, we first get the user and item latent embeddings $\mathbf{U}_i^B \in\mathbb{R}^{m\times h}$ and $\mathbf{V}_j^B \in\mathbb{R}^{n^2 \times h}$ with $h$ denoting the embedding dimension, and then the predicted ratings are the cosine similarities between user and item latent embeddings, i.e., $\hat{\mathbf{R}}^B_{ij} = (\mathbf{U}_i^B)^T \mathbf{V}_j^B$. 
Its loss function is given as:
\begin{equation}
\begin{aligned}
\mathcal{L}_{reg} = \sum_{i=1}^m \sum_{j=1}^{n_2} \left(
 \frac{\mathbf{R}^B_{ij}}{\max(\mathbf{R}^B)} log (\hat{\mathbf{R}}^B_{ij}) + (1-\frac{\mathbf{R}^B_{ij}}{\max(\mathbf{R}^B)}) log (1 - \hat{\mathbf{R}}^B_{ij})) \right) ,
\end{aligned}
\end{equation}
where $\max(\mathbf{R}^B)$ is the maximum rating in the target domain.
With the rating regression module, we can model the user-item rating interactions in the target domain.

\nosection{User embedding alignment between source and target domains}
Although the rating reconstruction module and the rating regression module can model the perturbed ratings of the source domain and the original ratings of the target domain respectively, the knowledge learnt from the source domain cannot be transferred to the target domain yet. 
To facilitate the knowledge transfer between the source and target domains, we further propose a user embedding alignment module. 
The user embedding alignment module aims to match the user embedding learnt from the rating reconstruction module and the user embedding learnt from the rating regression module. 
Our motivation is that users tend to have similar preferences across different domains. 
For example, users who like romantic movies are likely to like romantic book as well. 
Formally, the user embedding alignment loss is given as:
\begin{equation}
\begin{aligned}
\mathcal{L}_{ali} = \sum_{i=1}^m {\|\mathbf{U}_i^A - \mathbf{U}_i^B\|}^2 . 
\end{aligned}
\end{equation}

\nosection{Total loss}
The loss of \cdrmodel~is a combination of the three type of losses above. That is, 
\begin{equation}
\begin{aligned}
\mathcal{L} = \mathcal{L}_{rec} + \mathcal{L}_{reg} + \alpha \mathcal{L}_{ali},
\end{aligned}
\end{equation}
where $\alpha$ is a hyper-parameter that controls the strength of the user embedding alignment module. 


%% file: sections/analysis1.tex
\section{Analysis}\label{sec:ana}


%
%

\subsection{Privacy Analysis}

%
To analyze the privacy of Algorithm~\ref{alg:dp_rating_pub}, we study how the output distribution changes for neighbouring rating matrices and how to achieve both user-level DP and rating matrix DP. 
%
%
%

\nosection{User-level differential privacy}
We first present the user-level differential privacy, i.e., privacy loss for a single row in differential private published rating matrix. 
%
%
In Theorem~\ref{thm:user_dp}, we show that both JLT and SJLT achieve user-level DP. 
%

\begin{theorem} [User-level Differential Privacy]
\label{thm:user_dp}
Suppose $\mathbf{R}$ and $\mathbf{R'}$ are two neighbouring rating matrices as defined in Definition~\ref{def:neigh_rating}. 
Let $\mathbf{Y} \sim \mathbf{N}(0, 1)$ be a random vector. 
Define the two distributions of $\mathbf{R^{\intercal}Y}$ and $\mathbf{R'^{\intercal}Y}$ respectively as 
\begin{eqnarray*}
\pdf_{\mathbf{R}^{\intercal} \mathbf{Y}} & = & \frac{1}{\sqrt{ (2 \pi)^{n_1} \rm{det} (\mathbf{R}^{\intercal} \mathbf{R})}}
\exp \left( - \frac{1}{2} \mathbf{x}^{\intercal} \left( \mathbf{R^{\intercal} R} \right)^{-1} \mathbf{x} \right), \\
\pdf_{\mathbf{R'}^{\intercal} \mathbf{Y}} & = & \frac{1}{\sqrt{ (2 \pi)^{n_1} \rm{det} (\mathbf{R'}^{\intercal} \mathbf{R'})}}
\exp \left( - \frac{1}{2} \mathbf{x}^{\intercal} \left( \mathbf{R'^{\intercal} R'} \right)^{-1} \mathbf{x} \right). \\
\end{eqnarray*}
Fix $\epsilon_0 = \frac{\epsilon}{\sqrt{4r \ln(2/\delta)}}$ and $\delta_0 = \frac{\delta}{2r}$, and  
denote $\mathcal{S} = \{ x: e^{- \epsilon_0} \pdf_{\mathbf{R'^{\intercal} Y}} (x) \leq \pdf_{\mathbf{R^{\intercal} Y}}(x) \leq e^{\epsilon_0} 
\pdf_{\mathbf{R'^{\intercal} Y}} (x) \}$, 
then we have 
$\Pr[\mathcal{S}] \geq 1 - \delta_0$. 
Thus Algorithm~\ref{alg:dp_rating_pub} preserves $(\epsilon_0, \delta_0)-$DP for each row in the published rating matrix. 
\end{theorem}

\begin{proof}
We present its proof in Appendix~\ref{sec:append_user_dp}.
\end{proof}
 
%
%
%

%
%

\nosection{Rating matrix differential privacy}
Based on user-level privacy guarantees (Theorem~\ref{thm:user_dp}) and the k-fold composition theorem (Theorem~\ref{thm:k_fold_comp} in Appendix), we present differential privacy with respect to rating value of Algorithm~\ref{alg:dp_rating_pub}.

\begin{theorem}[Rating Matrix Differential Privacy] 
Algorithm~\ref{alg:dp_rating_pub} preserves $(\epsilon, \delta)-$differential privacy with respect to a single rating value changing in $\mathbf{R}$. 
\label{thm:dp_graph}
\end{theorem}

\begin{proof}
We present its proof in Appendix~\ref{sec:rating_dp_append}.
\end{proof}

By presenting privacy guarantees for user-level and rating matrix, we show our algorithm achieves \textbf{G1}: \textit{Privacy-preserving}. 
%
%

\subsection{Utility Analysis}
\label{sec:utility_ana}

We analyze the utility of Algorithm~\ref{alg:dp_rating_pub} from a probabilistic perspective. 
First, we show that the expectation of the perturbed published rating matrix on the low-dimensional subspace approximates the input rating matrix on the high-dimensional original space, with a bias determined by privacy parameters (Theorem~\ref{prop:mean_approx}). 
Based on the expectation approximation, we show Algorithm~\ref{alg:dp_rating_pub} achieves \textit{Restricted isometry property}, i.e., the output rating matrix concentrates around the mean approximation with a large probability determined by privacy parameters (Theorem~\ref{thm:RIP}). 

\nosection{Expectation approximation}
We first present the approximation effect of the published rating matrix to original one, which is the foundation of \text{restricted isometry property}. 
\begin{proposition}
\label{prop:mean_approx}
Let $\mathbf{R}$ and $\mathbf{\tilde{R}}$ be the input and output rating matrix of Algorithm~\ref{alg:dp_rating_pub}, respectively. 
Then, the mean squared error of the input and output covariance matrices are 
$$
\E \| \mathbf{R^{\intercal} R} - \mathbf{\tilde{R}^{\intercal} \tilde{R}} \|_2^2 \leq w^2 m 
= 16 n_1'^2 \ln(2/\delta) \ln^2(4 n_1'/\delta) / \epsilon^2 m, 
 $$
where 
$n_1' $ is the dimension of the reduced item space,  
$m $ is the number of user, and 
$(\epsilon, \delta)$ are privacy parameters. 
\end{proposition}
\begin{proof}
We present its proof in Appendix~\ref{sub:mean_append}.
\end{proof}

%
%
\nosection{Restricted isometry property}
We then come to the main results--- Algorithm~\ref{alg:dp_rating_pub} has \textit{Restricted isometry property}. 

\begin{theorem}
[Restricted Isometry Property]
\label{thm:RIP}
Let $\mathbf{R}$ and $\mathbf{\tilde{R}}$ be the input and output rating matrix of Algorithm~\ref{alg:dp_rating_pub}, respectively. 
Then 
\begin{align*}
& \Pr \left[  \left( 1 - \gamma \right)  \left(  \| \mathbf{R}   \|_F^2 + w^2 m  \right)   \leq \| \mathbf{\tilde{R} \|_F^2 }
\leq \left( 1 + \gamma \right)  \left(  \| \mathbf{R}   \|_F^2 + w^2 m  \right)  \right] \\
 & \leq  1 - 2 n_1'^{-2m},
\end{align*}
where $n_1'$  is the dimension of reduced item space,
$m$ is the number of user,
and $\gamma = O(\sqrt{\frac{\log m }{n_1'}}). $
\end{theorem}

\begin{proof}
We present its proof in Appendix~\ref{sub:RIP_append}.
\end{proof}

By the concentration results, we see that the larger subspace dimension $n_1'$, the larger probability
 $  \| \mathbf{\tilde{R}} \|_F^2 $ approximated $  \| \mathbf{R} \|_F^2  + w^2 m$ with a distance of $\epsilon$, which verifies the conclusion.

%
%
%
%
\nosection{Improvement of SJLT}
Furthermore, we explain how SJLT improves JLT from the perspective of preconditioning. 
Applying SJLT on a vector $\mathbf{x}$, we have $\mathbf{PHDx}$. 
%
%
Since $\mathbf{H}$ and $\mathbf{D}$ are unitary,  $\mathbf{HD}$ preserves $\ell_2-$norm, i.e., $\| \mathbf{HDx} \|_2 = \| \mathbf{x}\|_2$. 
\begin{lemma}[Randomized Hardmard transform~\cite{ailon2009fast}]
Let $\mathbf{H}$ and $\mathbf{D}$ be defined in Definition~\ref{def:fjlt}. For any set $V$ of $m$ vectors in $\mathbf{R}^{n_1}$, with probability at least $1 - 1/20$, 
$$
\max_{x \in V} \| \mathbf{HDx} \|_{\infty} \leq \left( \frac{2 \ln(40 m n_1)}{n_1} \right)^{1/2} \| \mathbf{x} \|_2. 
$$
\label{lem:rhd}
\end{lemma}
%

\begin{proof}
We present its proof in Appendix~\ref{sec:rhd_append}.
\end{proof}

By Lemma~\ref{lem:rhd}, we observe that the preconditioner $\mathbf{HD}$ reduces the $\ell_{\infty}-$norm of $\mathbf{x}$. 
With the same $\ell_2$-norm and reduced $\ell_{\infty}-$norm, the energy of vector $\mathbf{x}$ is spread out. 
That is, the preconditioner $\mathbf{HD}$ smoothes out the sparse vector $\mathbf{x}$ and improves the utility of sparse sub-Gaussian transform $\mathbf{P}$. 
To this end, our algorithm achieves \textbf{G2} (restricted isometry property) and \textbf{G3} (sparse-aware). 

%% file: sections/experiment.tex
\section{Experiments and Analysis}
\label{sec:experi}

In this section, we conduct experiments on two real-world datasets to answer the following research questions. 
\textbf{RQ1}: How can our model outperform the state-of-the-art recommendation model that is trained using the data of the target domain only?
\textbf{RQ2}: What is the performance of our model compared with the state-of-the-art CDR models that are trained using both the plaintext source and target datasets? 
\textbf{RQ3}: Is our proposed CDR model (\cdrmodel) effective for the heterogeneous user-item rating data in the source and target domains?
\textbf{RQ4}: How do the parameters of JLT and SJLT (mainly $\epsilon$, $n$, and $sp$) in stage 1 affect our model performance?
\textbf{RQ5}: Will different data sparsity in the source domain affect the performance of JLT and SJLT? 
%






\subsection{Experimental Setup} 

\nosection{Datasets} 
We choose two real-world datasets, i.e., Amazon and Douban. 
The \textbf{Amazon} dataset \cite{amazon,catn} has three domains, i.e., Movies and TV (Movie), Books (Book), and CD Vinyl (Music). 
The \textbf{Douban} dataset \cite{zhu2019dtcdr,zhu2021unified} also has three domains, i.e., Book, Music, and Movie. 
Since the user-item interaction ratings in both datasets range from 0 to 5, we binarize them by taking the ratings that are higher or equal to 3 as positive and others as negative, following \cite{hu2018conet}. 
We also filter the users and items with less than 5 interactions. 
For simplification, we only choose the user-item interactions of the common users across domains. 
The detailed statistics of these datasets after pre-process are shown in Appendix \ref{sec:app:data}.

\nosection{Comparison methods} 
To validate the performance of our proposed model, 
we compare \modelname~and its variants (\textbf{PriCDR-SYM}, \textbf{PriCDR-J}, and \textbf{PriCDR-S}) with both the famous single domain recommendation methods (\textbf{BPR}, \textbf{NeuMF}, and \textbf{DMF}) and the state-of-the-art CDR methods (\textbf{CoNet}, \textbf{DDTCDR}, \textbf{DARec}, and \textbf{ETL}). 
\textbf{BPR} \cite{rendle2012bpr} is a representative pairwise learning-based factorization model, focusing on minimizing the ranking loss between predicted ratings and observed ratings.
\textbf{NeuMF} \cite{he2017neural} is a representative neural network based model, replacing the conventional inner product with a neural architecture to improve recommendation accuracy.
\textbf{DMF} \cite{xue2017deep} is the state-of-the-art deep neural network based recommendation model, employing a deep architecture to learn the low-dimensional factors of users and items.
\textbf{CoNet} \cite{hu2018conet} enables dual knowledge transfer across domains by introducing cross connections unit from one base network to the other and vice versa.
\textbf{DDTCDR} \cite{li2020ddtcdr} introduces a deep dual transfer network that transfers knowledge with orthogonal transformation across domains.
\textbf{DARec} \cite{darec19} transfers knowledge between domains with shared users, learning shared user representations across different domains via domain adaptation technique.
\textbf{ETL} \cite{etl} is a recent state-of-the-art CDR model that adopts equivalent transformation module to capture both the overlapped and the non-overlapped domain-specific properties.  
\textbf{\modelname-SYM} is a comparative model for ablation study which replaces \cdrmodel~with a symmetrical CDR framework, using auto-encoder to model both the published source rating matrix and the target rating matrix. \modelname-SYM uses JLT based differential private rating matrix publishing mechanism.
\textbf{\modelname-J} is our proposed \modelname~with JLT based differential private rating matrix publishing mechanism in the first stage.
\textbf{\modelname-S} is our proposed \modelname~with Sparse-aware JLT based differential private rating matrix publishing mechanism.

\nosection{Evaluation method}
To evaluate the recommendation performance, We use the leave-one-out method which is widely used in literature \cite{he2017neural}.
That is, we randomly reserve two items for each user, one as the validation item and the other as the test item.
Then, following \cite{he2017neural,hu2018conet}, we randomly sample 99 items that are not interacted by the user as the negative items. 
We choose three evaluation metrics, i.e., Hit Ratio (HR), Normalized Discounted Cumulative Gain (NDCG), and Mean Reciprocal Rank (MRR), where we set the cut-off of the ranked list as 5 and 10.

\nosection{Parameter settings}
For both source domain and target domain, we use two-layer MLP with hidden size of 500 and 200 as the network architecture and ReLU as the activation function.
For a fair comparison, we choose Adam \cite{kingma2014adam} as the optimizer, and tune both the parameters of \modelname~and the baseline models to their best values.
For differential private related parameters, we vary $\epsilon$ in $\{0.5,1.0,2.0,4.0,8.0,16.0,32.0,64.0\}$ and subspace dimension for the differential privately published rating matrix $n$ in $\{100,200,300,400,500,600,700,800\}$. 
For SJLT, we vary the sparsity coefficient of the reduced matrix $sp$ in $\{0.1,0.3,0.5,0.7,0.9\}$. 
For CDR modeling, we set batch size to 128 for both the source and target domains. 
The dimension of the latent embedding is set to $h = 200$. 
Meanwhile we set the hyper-parameter $\alpha = 100$ for user embedding alignment, since it achieves the best performance. 
For all the experiments, we perform five random experiments and report the average results. 

\subsection{Model Comparison (\textbf{RQ1} and \textbf{RQ2})}

We report the comparison results on \textbf{Douban} and \textbf{Amazon} datasets in Table \ref{tab:compare}.
%
%
From the results, we can find that: (1) All the cross domain recommendations models  (e.g., \textbf{DARec}) perform better than single domain recommendation (e.g., \textbf{NeuMF}) models, indicating that cross domain recommendations can integrate more useful information to enhance the model performance, which always conform to our common sense.
(2) Our proposed \textbf{\modelname-J} and \textbf{\modelname-S} outperform all the single-domain and cross-domain baseline models in all tasks, which means the two-stage design of privacy-preserving CDR works well in predicting users' preference in target domain while preserving the data privacy of the source domain at the same time.
(3) By comparing the performance between \textbf{\modelname-J} and \textbf{\modelname-S}, we can conclude that Sparse-aware Johnson-Lindenstrauss Transform can effectively alleviate the data sparsity in the source domain, and thus significantly boosts the recommendation performance of \textbf{\modelname~}.

\begin{table*}[t]
\centering
\small
\caption{Experimental results on Amazon and Douban datasets. }
\vspace{-0.2cm}
\label{tab:compare}
\begin{tabular}{ccccccccccccc}
\toprule
\multirow{3}{*}{}
&HR@5 &NDCG@5 &MRR@5 &HR@10 &NDCG@10 &MRR@10 
&HR@5 &NDCG@5 &MRR@5 &HR@10 &NDCG@10 &MRR@10 \\
\cmidrule(lr){2-4} \cmidrule(lr){5-7} \cmidrule(lr){8-10} \cmidrule(lr){11-13}
& \multicolumn{6}{c}{(Amazon) Music$\rightarrow$Book}
& \multicolumn{6}{c}{(Amazon) Movie$\rightarrow$Music}\\

\midrule
BPR & 0.2633 & 0.2090 & 0.1827 & 0.3824 & 0.2580 & 0.2442 & 0.2448 & 0.1569 & 0.1734 & 0.3852 & 0.2035 & 0.1983 \\ 

NeuMF & 0.2760 & 0.2271 & 0.2086 & 0.4015 & 0.2577 & 0.2608 & 0.2937 & 0.1971 & 0.1988 & 0.4328 & 0.2401 & 0.2200 \\ 

DMF & 0.3075 & 0.2543 & 0.2367 & 0.4312 & 0.2713 & 0.2864 & 0.3316 & 0.2146 & 0.2266 & 0.4671 & 0.2715 & 0.2448 \\ 
\midrule 
CoNet & 0.3743 & 0.3157 & 0.3063 & 0.4831 & 0.3352 & 0.3245 & 0.3995 & 0.2857 & 0.2731 & 0.5328 & 0.3340 & 0.2973 \\ 

DDTCDR & 0.4067 & 0.3403 & 0.3241 & 0.5054 & 0.3690 & 0.3318 & 0.4311 & 0.3368 & 0.2974 & 0.5499 & 0.3736 & 0.3126 \\ 

DARec & 0.4741 & 0.3758 & 0.3420 & 0.5702 & 0.4098 & 0.3499 & 0.4916 & 0.3825 & 0.3402 & 0.6192 & 0.4220 & 0.3567 \\ 

ETL & 0.4769 & 0.3801 & 0.3462 & 0.5877 & 0.4153 & 0.3607 & 0.5308 & 0.4126 & 0.3714 & 0.6517 & 0.4511 & 0.3873 \\ 
\midrule 
\modelname-SYM & 0.4360 & 0.3467 & 0.3172 & 0.5487 & 0.3828 & 0.3321 & 0.4933 & 0.3815 & 0.3445 & 0.6103 & 0.4192 & 0.3600 \\
\modelname-J & \textbf{0.4898} & \textbf{0.3984} & \textbf{0.3648} & \textbf{0.5853} & \textbf{0.4292} & \textbf{0.3744} & \textbf{0.5399} & \textbf{0.4344} & \textbf{0.3992} & \textbf{0.6519} & \textbf{0.4706} & \textbf{0.4141} \\ 

\modelname-S & \textbf{0.5203} & \textbf{0.4264} & \textbf{0.3964} & \textbf{0.6181} & \textbf{0.4580} & \textbf{0.4092} & \textbf{0.5733} & \textbf{0.4608} & \textbf{0.4220} & \textbf{0.6853} & \textbf{0.4970} & \textbf{0.4392}  \\

\toprule

\multirow{3}{*}{}
& \multicolumn{6}{c}{(Amazon) Movie$\rightarrow$Book}
& \multicolumn{6}{c}{(Douban) Book$\rightarrow$Music}\\
\midrule
BPR & 0.2798 & 0.1665 & 0.2210 & 0.3716 & 0.2109 & 0.2392 & 0.1375 & 0.0822 & 0.0699 & 0.2889 & 0.1058 & 0.0950 \\ 

NeuMF & 0.3042 & 0.1989 & 0.2341 & 0.3999 & 0.2438 & 0.2403 & 0.1336 & 0.0913 & 0.0728 & 0.3042 & 0.1187 & 0.1093 \\ 

DMF & 0.3253 & 0.2281 & 0.2570 & 0.4238 & 0.2510 & 0.2439 & 0.1547 & 0.1003 & 0.0824 & 0.3315 & 0.1367 & 0.1100 \\ 
\midrule 
CoNet & 0.3744 & 0.2659 & 0.2978 & 0.4905 & 0.3050 & 0.3098 & 0.2044 & 0.1321 & 0.0996 & 0.3787 & 0.1633 & 0.1340 \\ 

DDTCDR & 0.4322 & 0.3490 & 0.3132 & 0.5288 & 0.3827 & 0.3355 & 0.2312 & 0.1425 & 0.1037 & 0.3934 & 0.1846 & 0.1506 \\ 

DARec & 0.4904 & 0.3824 & 0.3388 & 0.6034 & 0.4224 & 0.3552 & 0.2703 & 0.1878 & 0.1479 & 0.4285 & 0.2269 & 0.1711 \\ 

ETL & 0.5124 & 0.4119 & 0.3757 & 0.6339 & 0.4514 & 0.3915 & 0.3241 & 0.2161 & 0.1832 & 0.4545 & 0.2545 & 0.1987\\ 
\midrule 

\modelname-SYM & 0.4614 & 0.3665 & 0.3346 & 0.5801 & 0.4031 & 0.3496 & 0.3225 & 0.2113 & 0.1848 & 0.4535 & 0.2517 & 0.2002 \\
\modelname-J & \textbf{0.5320} & \textbf{0.4324} & \textbf{0.3992} & \textbf{0.6423} & \textbf{0.4673} & \textbf{0.4135} & \textbf{0.3290} & \textbf{0.2208} & \textbf{0.1854} & \textbf{0.4675} & \textbf{0.2643} & \textbf{0.2034} \\ 

\modelname-S & \textbf{0.5516} & \textbf{0.4483} & \textbf{0.4136} & \textbf{0.6585} & \textbf{0.4843} & \textbf{0.4285} & \textbf{0.3701} & \textbf{0.2582} & \textbf{0.2212} & \textbf{0.5097} & \textbf{0.3032} & \textbf{0.2397}  \\

\toprule

\multirow{3}{*}{}
& \multicolumn{6}{c}{(Douban) Movie$\rightarrow$Book}
& \multicolumn{6}{c}{(Douban) Movie$\rightarrow$Music}\\
\midrule
BPR & 0.1827 & 0.1132 & 0.1091 & 0.3436 & 0.1590 & 0.1375 & 0.1572 & 0.1079 & 0.0959 & 0.3387 & 0.1635 & 0.1134 \\ 

NeuMF & 0.1989 & 0.1279 & 0.1124 & 0.3640 & 0.1748 & 0.1453 & 0.1641 & 0.0932 & 0.0947 & 0.3462 & 0.1577 & 0.1204 \\ 

DMF & 0.2376 & 0.1438 & 0.1321 & 0.3952 & 0.1940 & 0.1511 & 0.1889 & 0.1104 & 0.1082 & 0.3754 & 0.1823 & 0.1409 \\ 
\midrule 
CoNet & 0.2736 & 0.1724 & 0.1533 & 0.4407 & 0.2330 & 0.1996 & 0.2284 & 0.1358 & 0.1305 & 0.4193 & 0.2179 & 0.1748 \\ 

DDTCDR & 0.2984 & 0.1803 & 0.1776 & 0.4699 & 0.2487 & 0.2174 & 0.2457 & 0.1639 & 0.1522 & 0.4298 & 0.2375 & 0.1846 \\ 

DARec & 0.3331 & 0.2285 & 0.2109 & 0.4970 & 0.2789 & 0.2310 & 0.2990 & 0.2017 & 0.1835 & 0.4576 & 0.2608 & 0.2081 \\ 

ETL & 0.3850 & 0.2705 & 0.2348 & 0.5247 & 0.3157 & 0.2572 & 0.3383 & 0.2357 & 0.2008 & 0.4739 & 0.2862 & 0.2286 \\ 
\midrule 
\modelname-SYM & 0.3742 & 0.2642 & 0.2310 & 0.5159 & 0.3094 & 0.2490 & 0.3194 & 0.2242 & 0.1930 & 0.4720 & 0.2737 & 0.2122 \\
\modelname-J & \textbf{0.3926} & \textbf{0.2788} & \textbf{0.2421} & \textbf{0.5349} & \textbf{0.3248} & \textbf{0.2603} & \textbf{0.3450} & \textbf{0.2496} & \textbf{0.2197} & \textbf{0.4872} & \textbf{0.2922} & \textbf{0.2373} \\ 

\modelname-S & \textbf{0.4301} & \textbf{0.3090} & \textbf{0.2702} & \textbf{0.5686} & \textbf{0.3503} & \textbf{0.2872} & \textbf{0.3801} & \textbf{0.2736} & \textbf{0.2440} & \textbf{0.5109} & \textbf{0.3158} & \textbf{0.2609}  \\
\bottomrule

\end{tabular}
\end{table*}

\subsection{In-depth Model Analysis}

\nosection{Ablation study (\textbf{RQ3})}
To study whether our model is effective for the heterogeneous user-item rating data, we compare \textbf{\modelname-J} and \textbf{\modelname-S} with \textbf{\modelname-SYM}. 
The results from Table \ref{tab:compare} show that both \textbf{\modelname-J}and \textbf{\modelname-S} outperform \textbf{\modelname-SYM} in all tasks, indicating that the heterogeneous CDR model we proposed in stage two makes it easier to handle the heterogeneity of rating data produced by stage one, and thus achieves a superior performance. 

\nosection{Parameter analysis (\textbf{RQ4})}
We now study the effects of hyper-parameters on model performance. The most important parameters in \textbf{\modelname-J} are the privacy parameter $\epsilon$ and subspace dimension $n$.
The key parameter in \textbf{\modelname-S} is $sp$, which denotes the sparsity degree of the  sparse Johnson-Lindenstrauss matrix. 
We report the results in Fig.~\ref{fig1:param}-Fig.~\ref{fig1:param-sp}, where we use \textbf{(Amazon) Music $\rightarrow$ Book} and \textbf{(Douban) Movie $\rightarrow$ Book} datasets.
Fig.~\ref{fig1:param} shows the effect of $\epsilon$. 
With user-level differential privacy, the published rating matrix changes with subspace dimension $n_1'$. For details, please refer to Appendix~\ref{sec:rating_dp_append}.
We can find that the performance gradually improves when $\epsilon$ increases and finally keeps a stable level after $\epsilon$ reaches 32.0 in \textbf{(Amazon) Music $\rightarrow$ Book}. And for the dataset \textbf{(Douban) Movie $\rightarrow$ Book}, the turning point is when $\epsilon=2.0$.
We also vary $n$ 
and report the results in Fig.~\ref{fig1:param-n}. The bell-shaped curve indicates that the accuracy will first gradually increase with $n$ and then slightly decrease, and \textbf{\modelname-J} achieves the best performance when $n=400$ in \textbf{(Amazon) Music $\rightarrow$ Book} and when $n=500$ in \textbf{(Douban) Movie $\rightarrow$ Book}.
Fig.~\ref{fig1:param-sp} shows the effect of $sp$. 
The results show that when $sp$ is small (e.g., 0.1), SJLT fails to grasp enough knowledge from the rating matrix in source domain. 
When $sp$ is large (e.g., 0.9), SJLT gets redundant information from source domain rating matrix. 
A medium $sp$, e.g., 0.5 on \textbf{(Amazon) Music $\rightarrow$ Book} and 0.7on \textbf{(Douban) Movie $\rightarrow$ Book}, enables SJLT to grasp enough knowledge from the source domain rating matrix while exert an implicit regularization on optimization, and thus can achieve the best performance. 
%

\begin{figure} 
    \centering
        \subfigure[Amazon Music $\rightarrow$ Book]{
    \begin{minipage}[t]{0.47\linewidth} 
    \includegraphics[width=4.2cm]{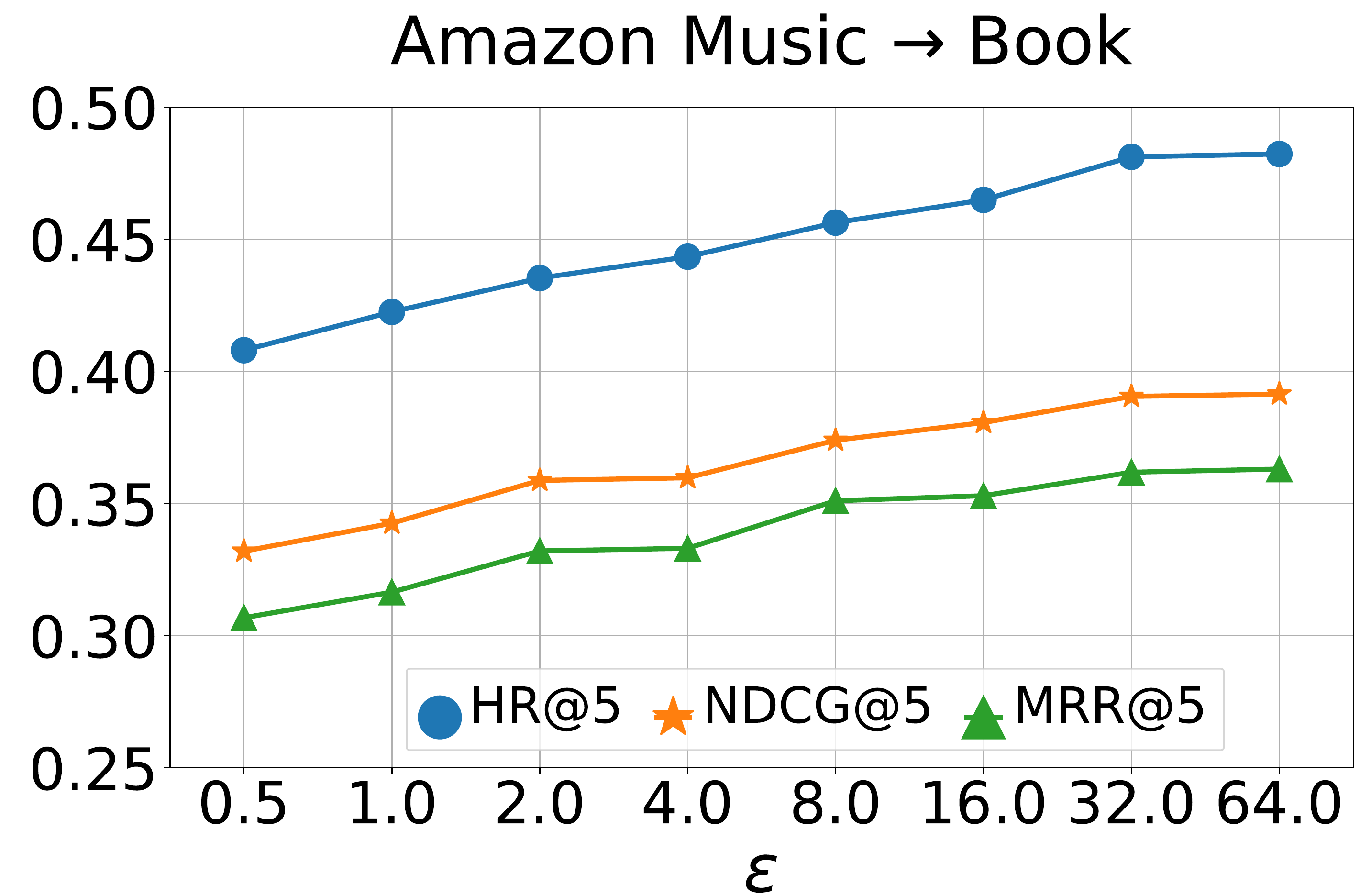}
    \end{minipage}
}
        \subfigure[Douban Movie $\rightarrow$ Book]{
    \begin{minipage}[t]{0.47\linewidth} 
    \includegraphics[width=4.2cm]{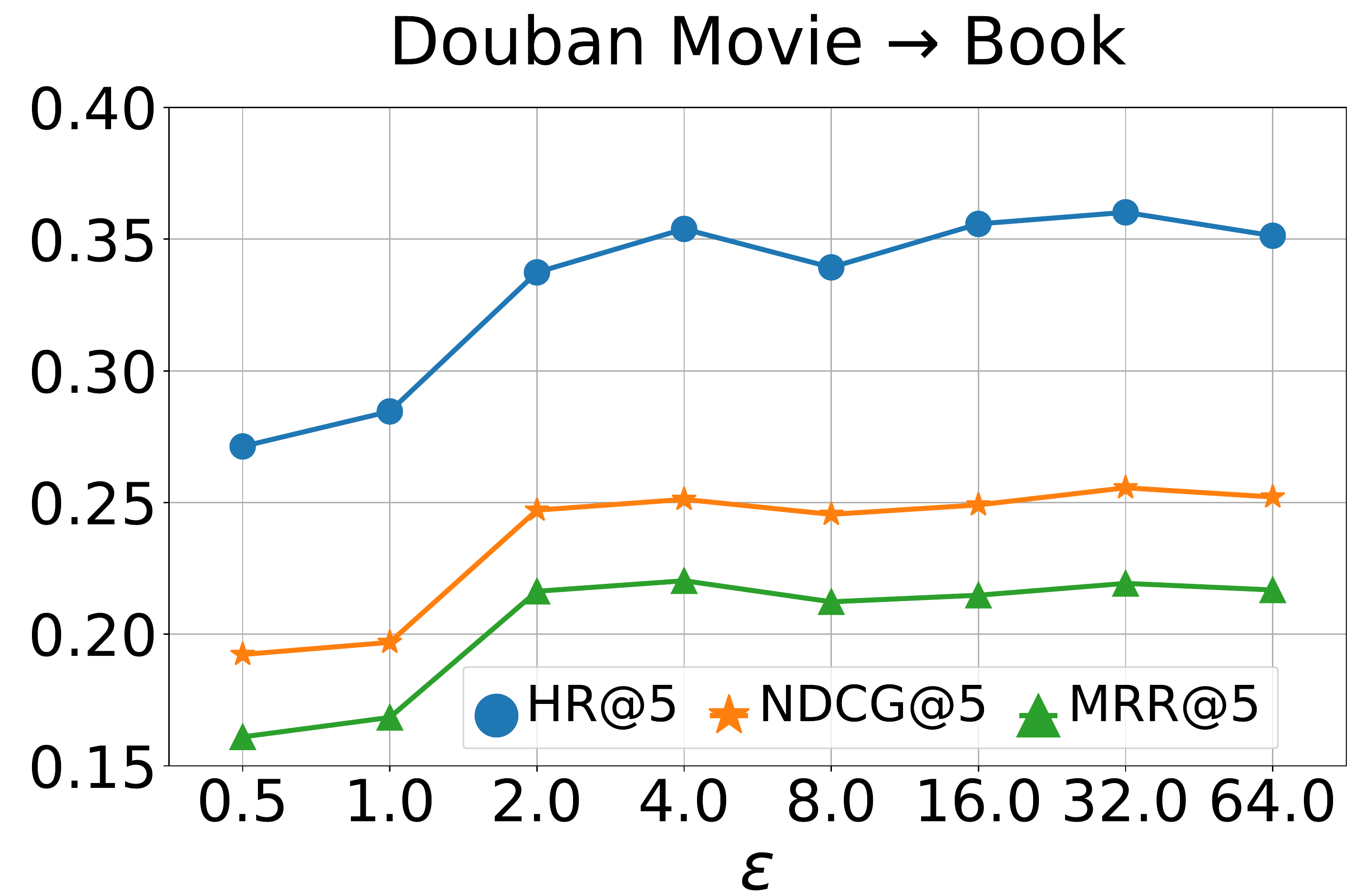}
    \end{minipage}
}
\vspace{-0.5cm}
	  \caption{Effect of the privacy parameter $\epsilon$ on \modelname-J. }
	  \vspace{-0.1cm}
	  \label{fig1:param}
\end{figure}

\begin{figure} 
    \centering
        \subfigure[Amazon Music $\rightarrow$ Book]{
    \begin{minipage}[t]{0.47\linewidth} 
    \includegraphics[width=4.2cm]{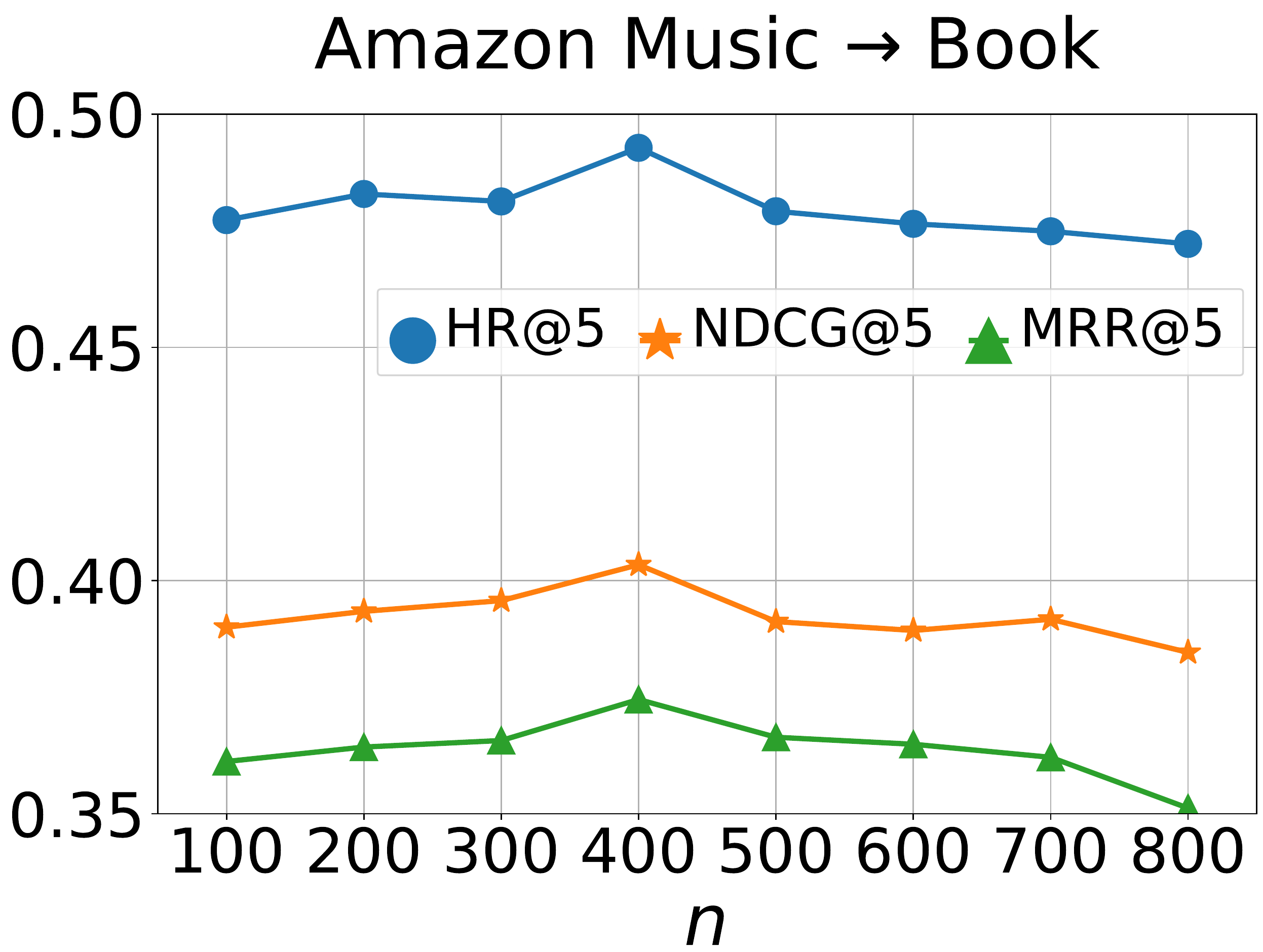}
    \end{minipage}
}
        \subfigure[Douban Movie $\rightarrow$ Book]{
    \begin{minipage}[t]{0.47\linewidth} 
    \includegraphics[width=4.2cm]{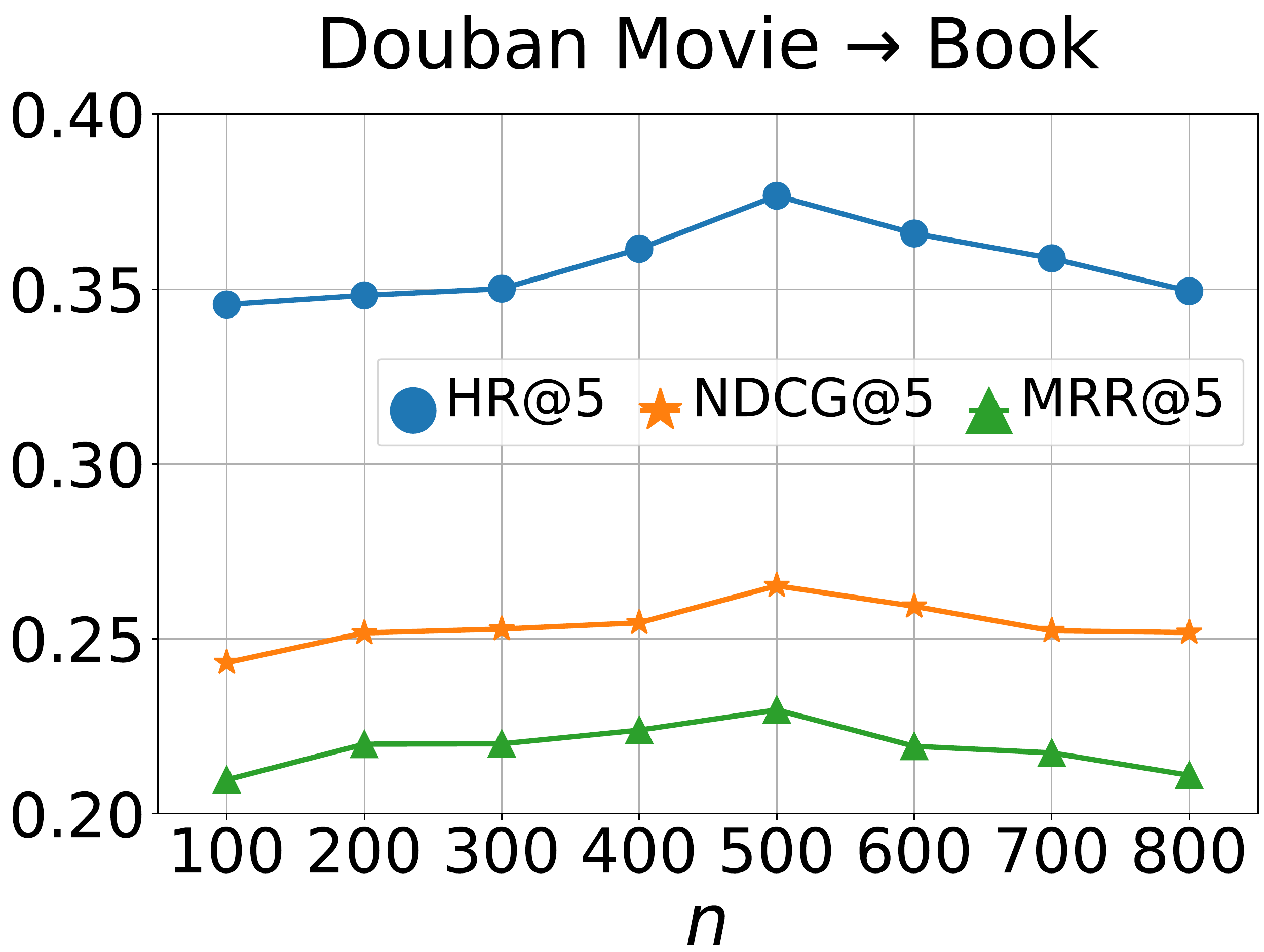}
    \end{minipage}
}
\vspace{-0.5cm}
	  \caption{Effect of subspace dimension $n$ on \modelname-J.}
	  \label{fig1:param-n}
\vspace{-0.1cm}
\end{figure}

\begin{figure} 
    \centering
        \subfigure[Amazon Music $\rightarrow$ Book]{
    \begin{minipage}[t]{0.47\linewidth} 
    \includegraphics[width=4.2cm]{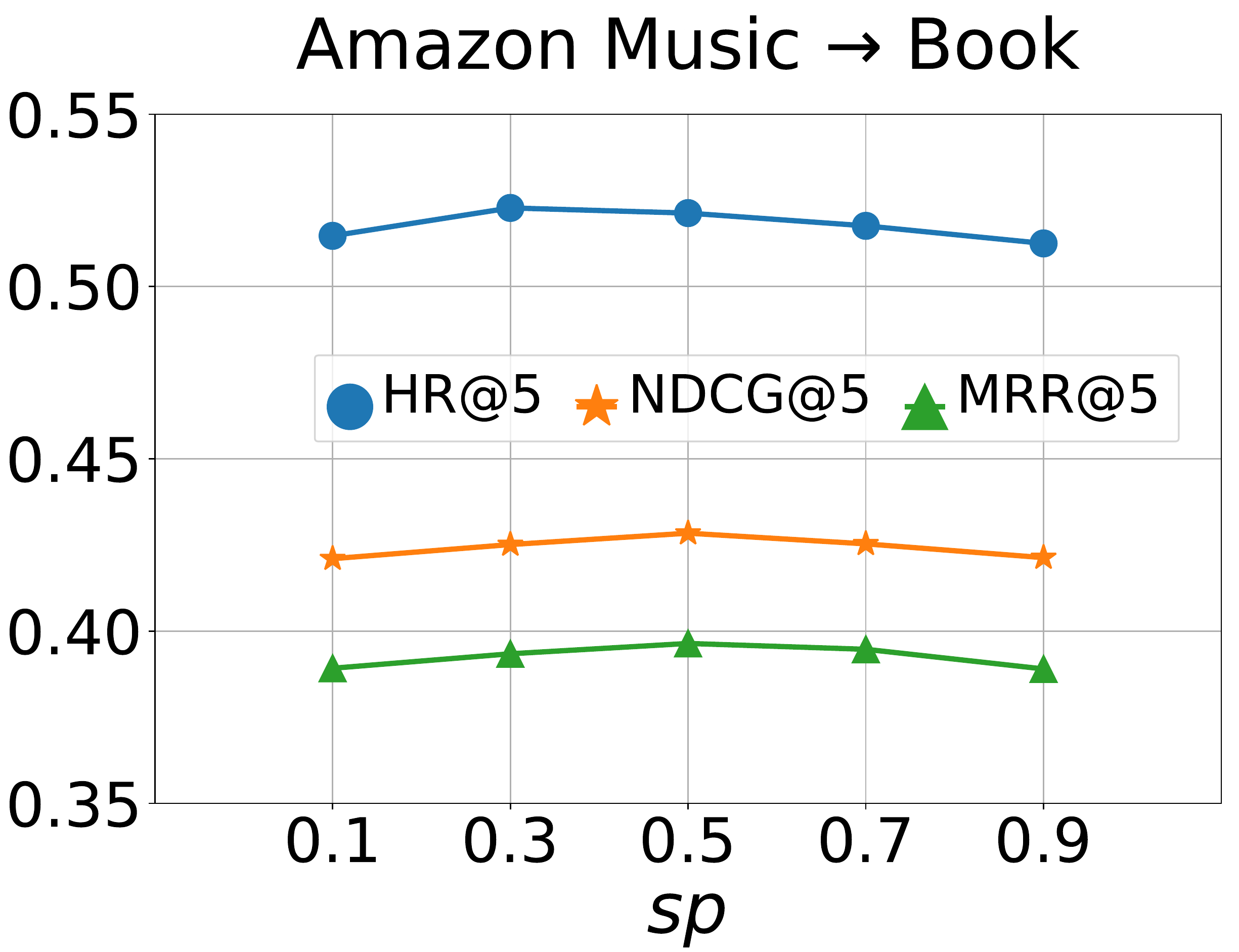}
    \end{minipage}
}
        \subfigure[Douban Movie $\rightarrow$ Book]{
    \begin{minipage}[t]{0.47\linewidth} 
    \includegraphics[width=4.2cm]{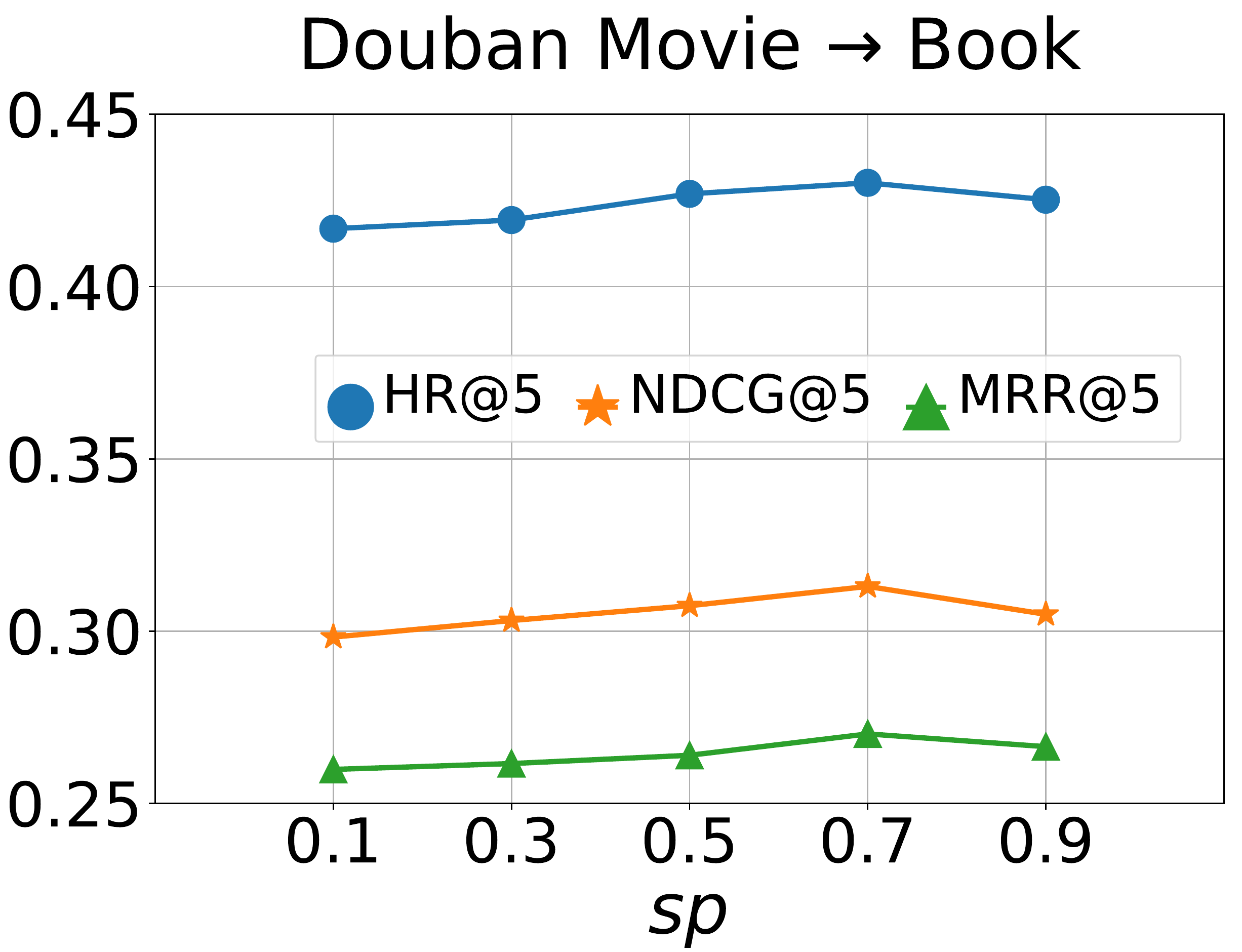}
    \end{minipage}
}
\vspace{-0.5cm}
	  \caption{Effect of sparsity param $sp$ on \modelname-S.  } 
	  \label{fig1:param-sp}
\vspace{-0.1cm}
\end{figure}

\begin{figure} 
    \centering
        \subfigure[Amazon Movie $\rightarrow$ Music]{
    \begin{minipage}[t]{0.47\linewidth} 
    \includegraphics[width=4.2cm]{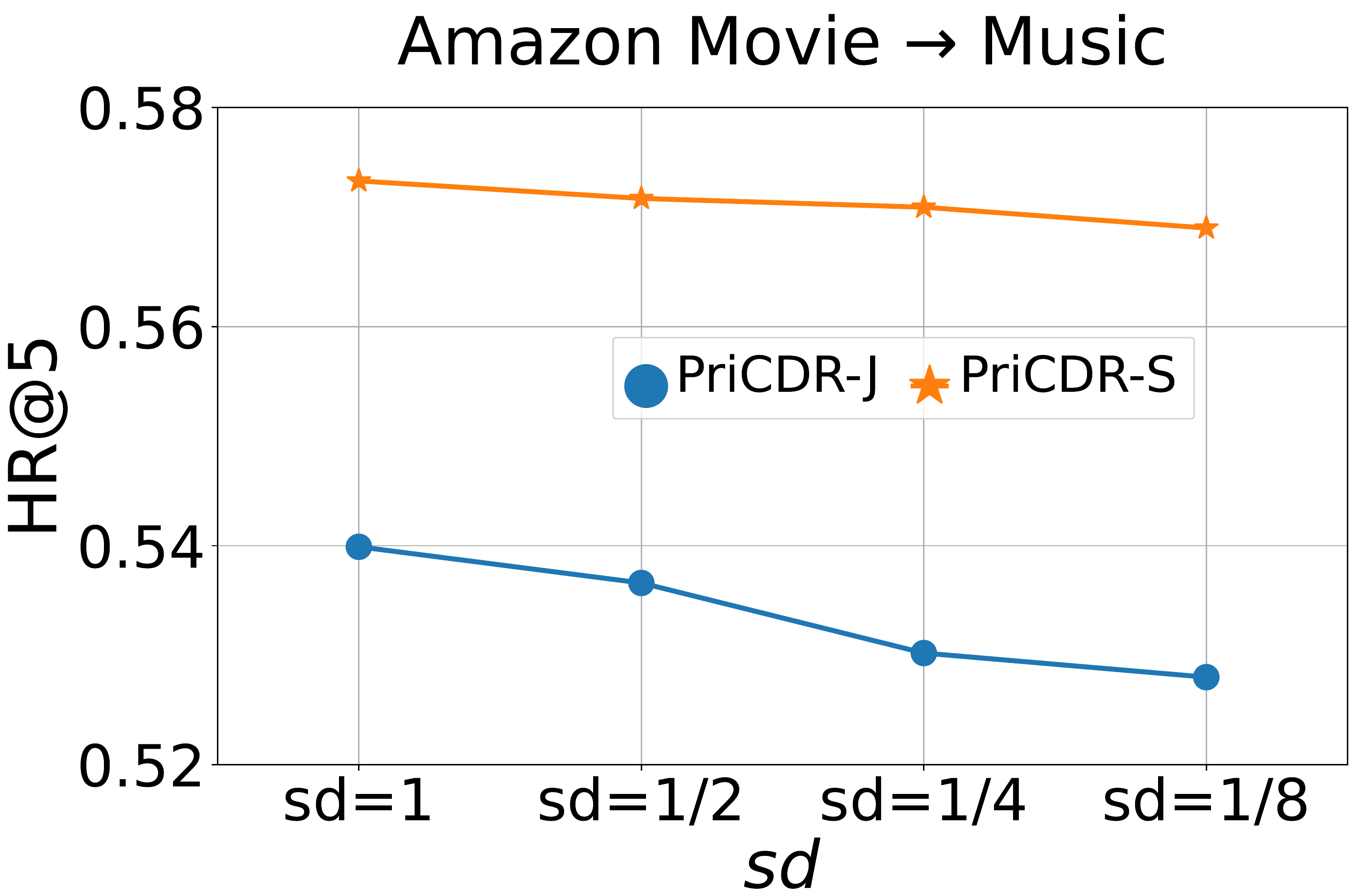}
    \end{minipage}
}
        \subfigure[Douban Movie $\rightarrow$ Book]{
    \begin{minipage}[t]{0.47\linewidth} 
    \includegraphics[width=4.2cm]{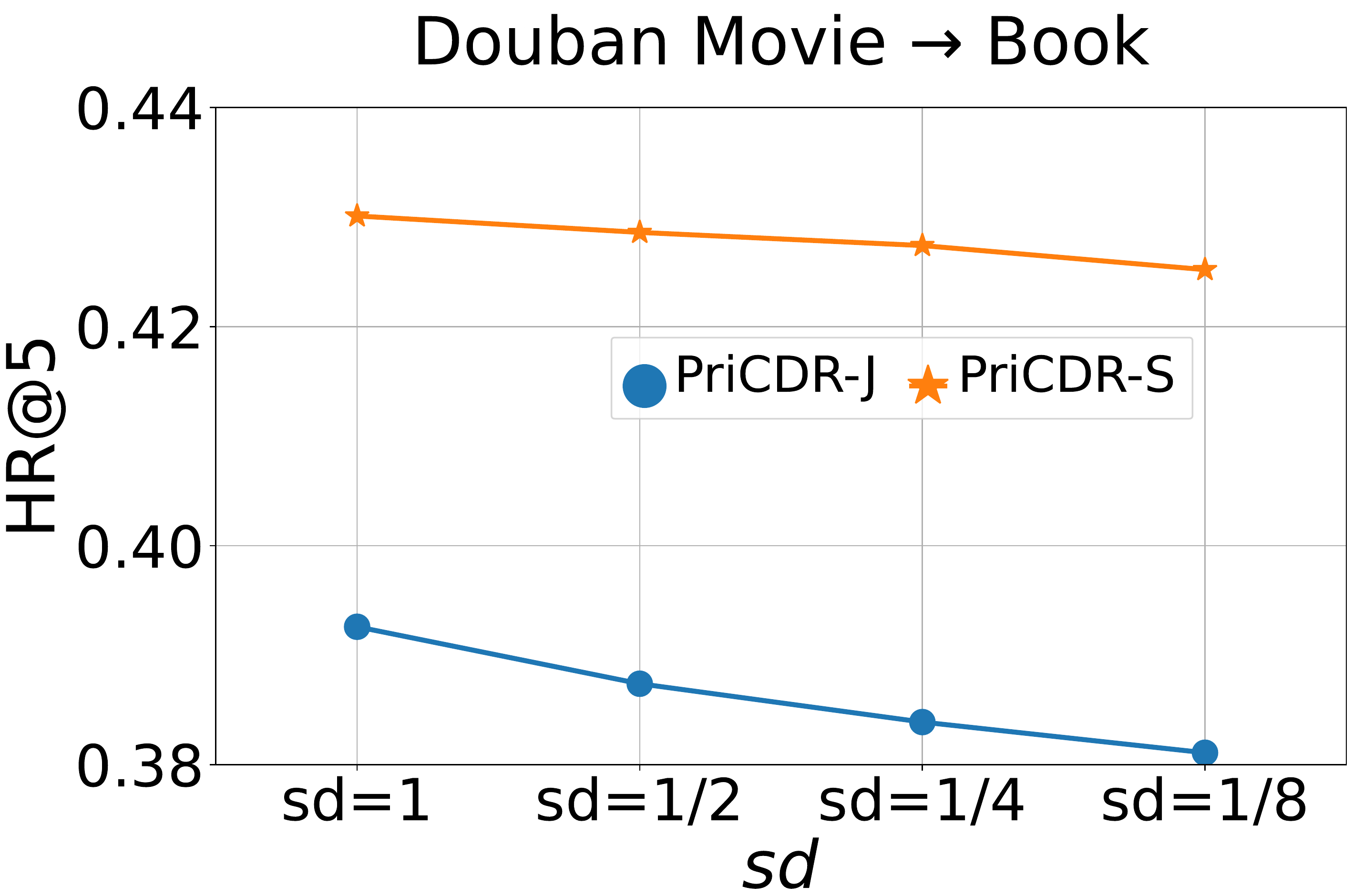}
    \end{minipage}
}
\vspace{-0.5cm}
	  \caption{Effect of the source domain sparsity degree $sd$. 
	  }
	  \label{fig1:param-sd}
\vspace{-0.1cm}
\end{figure}




\nosection{Effect of source domain sparsity (\textbf{RQ5})}
To study the effect of source domain sparsity on the performance of JLT and SJLT, we change the sparsity of the source domain by sampling from the original dataset, 
such that the sparsity degree $sd$ could be adjusted to $\{1,1/2,1/4,1/8\}$. Here, $sd=1$ denotes the original dataset without sampling.
We report the results in Fig.~\ref{fig1:param-sd}, where we use \textbf{(Amazon) Movie $\rightarrow$ Music} and \textbf{(Douban) Movie $\rightarrow$ Book} datasets.
From the results, we can conclude that:
(1) The decrease of $sp$ brings the decrease of the performance for both JLT and SJLT.
(2) SJLT shows greater stability than JLT when the source domain becomes more sparse, which is owning to its sparse-aware ability, as we have analyzed in Sec.~\ref{sec:utility_ana}.

%% file: sections/conclusion.tex
\section{Conclusion}
In this paper, we aim to solve the privacy issue in existing Cross Domain Recommendation (CDR) models. 
For this, we propose a novel two stage based privacy-preserving CDR framework, namely \modelname. 
In \textit{stage one}, the source domain privately publishes its user-item ratings to the target domain, and we propose two methods, i.e., Johnson-Lindenstrauss Transform (JLT) based and Sparse-aware JLT (SJLT) based, for it.  
We theoretically analyze the privacy and utility of our proposed differential privacy based rating publishing methods. 
In \textit{stage two}, the target domain builds a CDR model based on its raw data and the published data of the source domain, and we propose a novel heterogeneous CDR model (\cdrmodel) for it. 
We empirically study the effectiveness of our proposed \modelname~and \cdrmodel~on two benchmark datasets and the comprehensive experimental results show their effectiveness. 

%% file: sections/appendix1.tex
\section{Notations}\label{sec:app:note}
We summarize the main notations used in this paper in Table \ref{tab:notation}.

\begin{table}[htbp]
  \centering
  \caption{Notation Table.}
    \begin{tabular}{cc}
    \hline
    Notation & Meaning \\
    $\epsilon, \delta$ & privacy hyper-parameters in DP \\
    $\eta, \mu$ & utility hyper-parameters in DP \\
    $\mathbf{O}$ & matrix of all elements equal one \\
    $\mathbf{I}$ & identity matrix \\ 
    $\mathbf{H}$ & normalized Hadamard matrix \\
    $\mathbf{D}$ & randomized diagonal matrix\\
    $\mathbf{M}$ & Gaussian random matrix \\
    $\mathbf{P}$ & sub-Gaussian random matrix \\
    \hline
    \end{tabular}%
  \label{tab:notation}%
\end{table}%

%
%
\section{Privacy Analysis}

\subsection{User-level Differential Privacy}
\label{sec:append_user_dp}

\begin{proof}
To achieve privacy-preserving property, we need to verify the user-level differential privacy as JLT described in Theorem~\ref{thm:user_dp}. 
As defined in Definition~\ref{def:neigh_rating}, we suppose two neighbouring rating matrix $\mathbf{R}$ and $\mathbf{R'}$ with a rank-1 gap, i.e., $\mathbf{E} = \mathbf{R} - \mathbf{R'} = \mathbf{e_i v^{\intercal}}. $
Suppose the SVD decomposition of $\mathbf{R}$ and $\mathbf{R'}$ are $\mathbf{R} = \mathbf{U \Sigma V^{\intercal}}$ and 
$\mathbf{R'} = \mathbf{U' \Lambda V'^{\intercal}}. $
Due to the limited space, we present the proof for the scenario that both singular values of $\mathbf{R}$ and $\mathbf{R'}$ are greater than the perturbation parameter $w$. The general case is an extension of this proof combined with the results in ~\cite{blocki2012johnson}. 
To achieve the user-level differential privacy of JLT, we considere the difference between the two distributions as 
\begin{align*}
\pdf_{\mathbf{R}^{\intercal} Y} (x) & = \frac{1}{\sqrt{(2 \pi)^d \det(\mathbf{R^{\intercal} R})}} \exp \left( - \frac{1}{2} \mathbf{x} \left( \mathbf{R^{\intercal} R} \right)^{-1}  \mathbf{x} \right), \\
\pdf_{\mathbf{R'}^{\intercal} Y} (x) & = \frac{1}{\sqrt{(2 \pi)^d \det(\mathbf{R'^{\intercal} R'})}} \exp \left( - \frac{1}{2} \mathbf{x} \left( \mathbf{R'^{\intercal} R'} \right)^{-1}  \mathbf{x} \right),
\end{align*}
where $\mathbf{Y} \sim \mathcal{N} (0, \mathbf{I}_{n \times 1})$ and $\mathbf{x}$ is a random vector sampled from $\mathbf{R^{\intercal} Y}. $
Then we have two stage results. Firstly,
\begin{equation}
\label{app:svd_1}
e^{- \epsilon_0/2} \leq \sqrt{\frac{\det(\mathbf{R'^{\intercal} R)}}{\det(\mathbf{R^{\intercal} R })}} \leq e^{\epsilon_0/2}. 
\end{equation}
Secondly, 
$$
\Pr \left[ \frac{1}{2} | \mathbf{x^{\intercal}} \left(
\left( \mathbf{R^{\intercal} R} \right)^{-1} - 
\left( \mathbf{R'^{\intercal} R'} \right)^{-1} \right) \mathbf{x} |
\geq \epsilon_0 / 2 \right] \leq \delta_0.  
$$
By Lindskii's theorem on rank-1 perturbation (Theorem 9.4 in \cite{bhatia2007perturbation}), we have 
$ e^{- \epsilon_0/2} \leq \sqrt{\prod_{i} \frac{\lambda_i^2}{\sigma_i^2}} \leq e^{\epsilon_0/2}$, and thus Equation~\eqref{app:svd_1} naturally holds. 
For the second result, by the discussion in~\cite{blocki2012johnson}, we have 
$$
| \mathbf{x}^{\intercal} \left( \mathbf{R^{\intercal} R} \right)^{-1} \mathbf{x} - \mathbf{x}^{\intercal} \left( \mathbf{R'^{\intercal} R'} \right)^{-1} \mathbf{x} | \leq 
2 \left( \frac{1}{w} + \frac{1}{w^2} \right) \ln (4/\delta_0). 
$$
To achieve $(\epsilon_0, \delta_0)-$DP on user-level, we solve the following inequality exactly
$$
2 \left( \frac{1}{w} + \frac{1}{w^2} \right) \ln(4/\delta_0) \leq \epsilon_0. 
$$
We have 
$$
- \sqrt{\frac{\epsilon_0}{2 \ln (4/\delta_0)} + \frac{1}{4}} 
\leq \frac{1}{w} + \frac{1}{2} 
\leq \sqrt{\frac{\epsilon_0}{2 \ln (4/\delta_0)} + \frac{1}{4}}.
$$
Since $w > 0$, we have 
$$
w \geq \frac{1}{\sqrt{\epsilon_0}{2 \ln (4/\delta_0)} + \frac{1}{4} - \frac{1}{2}}. 
$$
Thus, by choosing 
$w = \frac{1}{\sqrt{\epsilon_0}{2 \ln (4/\delta_0)} + \frac{1}{4} - \frac{1}{2}},$
we have Algorithm~\ref{alg:dp_rating_pub} with JLT satisfies $(\epsilon_0, \delta_0)$ user-level DP. 
Then we discuss the case of Algorithm~\ref{alg:dp_rating_pub} with SJLT. 
By the construction of SJLT in Definition~\ref{def:fjlt}, we have two good properties for Algorithm~\ref{alg:dp_rating_pub}. 
\begin{itemize}
    \item $\mathbf{D}$ and $\mathbf{H}$ are unitary, i.e., 
    $
    \mathbf{D}^{\intercal} \mathbf{D} = \mathbf{I}; \quad 
    \mathbf{H}^{\intercal} \mathbf{H} = \mathbf{I}, 
    $
    \item Rows in $\mathbf{P}$ are sub-Gaussian, 
\end{itemize}

where $\mathbf{P}$ is constructed by concatenating $n'_1$ i.i.d. sparse random Gaussian vectors $\mathbf{Y}$ which is generated by
$$
\mathbf{Y}_i = 
\begin{cases}
0, \quad & \quad \text{with probability } 1-q; \\
\xi \sim \GN(0, q^{-1})& \quad \text{with probability } q. 
\end{cases}
$$

In order to check the distribution between $\mathbf{\tilde{R}}$ and $\mathbf{\tilde{R}}'$, we compare the difference in Probability Density Function (PDF) of a single row in $\mathbf{\tilde{R}}$ and $\mathbf{\tilde{R}'}$, i.e., 
$\mathbf{R}^{\intercal} \mathbf{Y}$ and ${\mathbf{R}'}^{\intercal} \mathbf{Y}$,  with $\mathbf{Y}$ denoting a sub-Gaussian vector.
Then the privacy of Algorithm~\ref{alg:dp_rating_pub} with SJLT follows the same procedure as JLT. 

\end{proof}

%
%

\subsection{Rating Matrix Differential Privacy}
\label{sec:rating_dp_append}

By user-level differential privacy, we then compare the distribution difference on probability density function of a single row in output rating matrix $\mathbf{\tilde{R}}$ and $\mathbf{\tilde{R}}'$. 
To analyze the difference of two whole rating matrices $\mathbf{\tilde{R}}$ and $\mathbf{\tilde{R}}'$, we utilize the following composition theorem presented in Theorem~\ref{thm:k_fold_comp}. 

\nosection{Composition Theorem}
By user-level differential privacy, we then compare the distribution difference on probability density function of a single row in output rating matrix $\mathbf{\tilde{R}}$ and $\mathbf{\tilde{R}}'$. 
To analyze the difference of two whole rating matrices $\mathbf{\tilde{R}}$ and $\mathbf{\tilde{R}}'$, we utilize the following composition theorem presented in Theorem~\ref{thm:k_fold_comp}. 

\begin{theorem} [Composition Theorem \cite{dwork2010boosting}]
For every $\epsilon >0, \delta, \delta' >0$, and $k \mathbb{N},$ the class of $(\epsilon, \delta)-$differential private mechanisms is 
$(\epsilon', k \delta + \delta')-$differentially private under $k-$fold adaptive composition, for 
\begin{equation}
\label{eqn:comp_k}
\epsilon' = \sqrt{2k \ln (1/\delta')} \epsilon + k \epsilon \epsilon_0,
\end{equation}
where $\epsilon_0 = e^{\epsilon} -1.$
\label{thm:k_fold_comp}
\end{theorem}

\nosection{Proof of rating matrix differential privacy}
We simply describe the proof as follows. 

\begin{proof}
By Theorem~\ref{thm:user_dp}, we have  Algorithm~\ref{alg:dp_rating_pub} preserves $(\epsilon_0, \delta_0)-$DP on user-level. 
Let $\epsilon_0 = \frac{\epsilon}{\sqrt{4 n_1' \ln (2/\delta)}}$ and $\delta_0 = \frac{\delta}{2 n_1'}$, and 
plugging them in Equation~\ref{eqn:comp_k}, we have Algorithm~\ref{alg:dp_rating_pub} preserves $(\epsilon, \delta)-$DP for rating matrix. 
\end{proof}

%
%

\section{Utility Analysis}

%
%
\subsection{Expectation Approximation}
\label{sub:mean_append}

\begin{proof}
Let $\mathbf{R}$ be the input rating matrix. 
The SVD decomposition of $\mathbf{R}$ is 
$$
\mathbf{R} = \mathbf{U D V^{\intercal}}. 
$$
After the singular value perturbation, we have 
$$
\mathbf{R_1} = \mathbf{U \sqrt{D^2 + w^2 I} V^{\intercal}}. 
$$
Then the output rating matrix is 
$$
\mathbf{\tilde{R}} = \frac{1}{n_1'} \mathbf{M R_1},  
$$
where $\mathbf{M}$ is the random matrix generated by Definition~\ref{def:jlt} or \ref{def:fjlt}. 

Then we consider the expectation mean squared error of covariance matrices, 
\begin{align*}
\E [\mathbf{\tilde{R}^{\intercal} \tilde{R} } ] 
&= \E [ \mathbf{R_1^{\intercal} M^{\intercal} M R_1}] 
= \mathbf{R_1^{\intercal} \E[M^{\intercal} M] R_1 }  
= \mathbf{ R_1^{\intercal} R_1 } \\
& = \mathbf{V \sqrt{D^2 + w^2 I} U^{\intercal} U \sqrt{D^2 + w^2 I} V^{\intercal}} 
= \mathbf{V (D^2 + w^2 I) V^{\intercal}}. 
\end{align*}

The last equality is confirmed by the property of Gaussian orthogonal ensemble, i.e., 
$$
\E[M_{ij} M_{jk}] = \delta_{jk}. 
$$
Thus
$$
\| \mathbf{R^{\intercal} R} - \E [\mathbf{\tilde{R}^{\intercal} \tilde{R} } ] \|_2^2 \leq
w^2 m 
= 16 n_1'^2 \ln(2/\delta) \ln^2(4 n_1'/\delta) / \epsilon^2 m ,
$$ 
with $n_1' = 8 \ln(2/\mu) / \eta^2$. 

\end{proof}

Before prove RIP (Theorem~\ref{thm:RIP}), we present the proposition used in proving Johnson-Lindenstrauss Lemma~\cite{dasgupta2003elementary}. 
\begin{proposition}
\label{prop:jl}
Let $\mathbf{u} \in \mathbf{R}^{n_1}$ and $\mathbf{M} \in \mathbf{R}^{n_1 \times n_1'}$ be either JLT or SJLT. 
Let $\mathbf{v} = \frac{1}{\sqrt{n_1'}} \mathbf{M u}$, 
and $\gamma = O(\sqrt{\frac{\log m}{n_1'}}). $
Then 
$$\Pr \left( \| \mathbf{v} \|_2^2 \geq (1 + \gamma) \| \mathbf{u} \|^2_2 \right) \leq n_1'^{-2}. 
$$
\end{proposition}

%
%

\subsection{Restricted Isometry Property}
\label{sub:RIP_append}

\begin{proof}
Let $\mathbf{R_{i}}$ be the column vector of rating matrix $\mathbf{R}$ for $i = 1, \cdots, m$ respectively. 
Each $\mathbf{R_{i}}$ represents the user-item rating list for user $i$. 
%
Let 
$\mathbf{v} = \mathbf{\tilde{R}_{i}} 
 = \mathbf{U}  \sqrt{ \mathbf{D}^2 + w^2 \mathbf{I}} \mathbf{ V^{\intercal}_{i}}.$ Then 
 $$
 \| \mathbf{v} \|_2^2 =
 \mathbf{v^{\intercal}} \mathbf{v} 
 = \mathbf{ V_{i}^{\intercal}} (\mathbf{D}^2 + w^2 \mathbf{I})  \mathbf{V_{i} }. 
 $$
 By Proposition~\ref{prop:jl}, we have 
 $$
\Pr \left(  \| \mathbf{\tilde{R}_{i}} \|_2^2 \geq \left( 1 + \gamma \right) \left( \|  \mathbf{R_{i}} \|_2^2 + w^2   \right)  \right) 
\leq n_1'^{-2}. 
 $$
 Denote the event $S_i = \{ \| \mathbf{\tilde{R}_i } \|_2^2   \geq  \left( 1 + \gamma \right) \left( \|  \mathbf{R_{i}} \|_2^2 + w^2  \right)  \}$
 and $S = \{ \| \mathbf{\tilde{R}} \|_F^2 \geq \left( 1 + \epsilon \right) \left( \| \mathbf{R} \|_F^2 + w^2 m \right)  \}.  $
By the fact 
$S \subset \bigcap_{i=1}^m S_i ,$
and each row of $\mathbf{\tilde{R}}$ are pairwise independent. We have 
$$
\Pr [S] \leq \prod_{i=1}^m \Pr[S_i] \leq   n_1'^{-2m}. 
$$
To conclude, we have 
$$
\Pr [\| \mathbf{\tilde{R}} \|_F^2 \geq \left( 1 + \gamma \right) \left(  \| \mathbf{R}   \|_F^2 + w^2 m  \right)] \leq {n_1'}^{-2m}.
$$
Similarly, 
$$
\Pr [\| \mathbf{\tilde{R}} \|_F^2 \leq \left( 1 - \gamma \right) \left(  \| \mathbf{R}   \|_F^2 + w^2 m  \right) ] \leq {n_1'}^{-2m}.
$$
Thus 
\begin{align*}
& \Pr \left[  \left( 1 - \gamma \right)  \left(  \| \mathbf{R}   \|_F^2 + w^2 m  \right)   \leq \| \mathbf{\tilde{R} \|_F^2 }
\leq \left( 1 + \gamma \right)  \left(  \| \mathbf{R}   \|_F^2 + w^2 m  \right)  \right] \\
 & \leq  1 - 2 n_1'^{-2m}.
\end{align*}

 \end{proof}

%
%

\subsection{Preconditioning Effect of Randomized Hardmard Transform}
\label{sec:rhd_append}

\begin{proof}
Without loss of generality,we supoose $\| \mathbf{x}\|_2 = 1$. 
Denote $\mathbf{u} = \begin{bmatrix} u_1, \cdots u_{n_1} \end{bmatrix}^{\intercal} = \mathbf{HD} \begin{bmatrix} x_1, \cdots, x_{n_1} \end{bmatrix}^{\intercal}. $
For each $j$, 
$u_j = \sum_{i=1}^{n_1} \mathbf{HD}_i x_i$, where 
$\mathbf{HD}_i = \pm \frac{1}{\sqrt{n}_1}. $
Then we can compute the moment generating function of $u_j$, 
$$
\E [e^{t n_1 u_j}] = \prod_{i=1}^{n_1} \E [e^{t n_1 \mathbf{HD}_i x_i}] =  \prod_{i=1}^{n_1} \frac{1}{2} \left( e^{t x_i \sqrt{n_1}} + e^{-t x_i \sqrt{n_1}} \right) \leq e^{t^2 n_1}/2. 
$$
By the results of Chernoff tail bound~\cite{enwiki:1050419718}, we have
$$
\Pr \left( |u_i | \geq a \right) 
\leq 2 e^{- a^2 n_1/2} \leq \frac{1}{20 n_1 m}. 
$$
Let $a^2 = 2 \ln (40 n_1 m)/n_1$, and sum over $n_1 m$ coordinates, we have the desired result. 
\end{proof}


\section{Datasets Description}\label{sec:app:data}

%
We conduct extensive experiments on two popularly used real-world datasets, i.e., \textit{Amazon} and \textit{Douban}, for evaluating our model on CDR tasks.
First, the \textbf{Amazon} dataset has three domains, i.e., Movies and TV (Movie), Books (Book), and CDs and Vinyl (Music) with user-item ratings.
Second, the \textbf{Douban} dataset has three domains, i.e., Book, Movie and Music. 
We show the detailed statistics of these datasets after pre-process in Table 3.

\begin{table}[htbp]
  \centering
  \caption{Dataset statistics of of Amazon and Douban.}
    \begin{tabular}{cccccc}
    \hline
    
    \multicolumn{2}{c}{ \textbf{Datasets} } & \textbf{Users} & \textbf{Items} &  \textbf{Ratings} & \textbf{Density}\\
    
    \hline
    
    \multirow{2}{*}{\textbf{Amazon}} & Music & \multirow{2}{*}{16,367} &   18,467   &  233,251   & 0.08\% \\
          & Book &       &  23,988   &   291,325   & 0.07\% \\

\hline

    \multirow{2}{*}{\textbf{Amazon}} & Movie & \multirow{2}{*}{15,914} &  19,794     &  416,228     & 0.13\% \\
         & Music  &       &  20,058     &  280,398     & 0.08\% \\
    
    \hline
    
    \multirow{2}{*}{\textbf{Amazon}} & Movie & \multirow{2}{*}{29,476} & 24,091      &  591,258     & 0.08\%  \\
         & Book  &       &   41,884    &   579,131    & 0.05\% \\
    
    \hline
    
    \multirow{2}{*}{\textbf{Douban}} & Book & \multirow{2}{*}{924} &    3,916   &    50,429   &  1.39\%\\
          & Music  &       &   4,228    &   50,157    & 1.28\% \\
    
    \hline
    
    \multirow{2}{*}{\textbf{Douban}} & Movie & \multirow{2}{*}{1,574} &   9,471    &    744,983   &  4.99\%\\
          & Book  &       &   6,139    &   85,602    & 0.89\% \\
    
    \hline
    
    \multirow{2}{*}{\textbf{Douban}} & Movie & \multirow{2}{*}{1,055} &   9,386    &  557,989     &  5.63\%\\
          & Music  &       &   4,981    &   60,626    & 1.15\% \\
\hline

    \end{tabular}%
  \label{tab:datasetss}%
\end{table}%
